\documentclass[runningheads]{llncs}

\usepackage{amsmath,amsfonts,bm}

\def\1{\bm{1}}

\DeclareMathAlphabet{\mathsfit}{\encodingdefault}{\sfdefault}{m}{sl}
\SetMathAlphabet{\mathsfit}{bold}{\encodingdefault}{\sfdefault}{bx}{n}

\usepackage{multirow}
\usepackage{multicol}

\usepackage{eccv}

\usepackage{eccvabbrv}

\usepackage{enumitem}
\usepackage{tcolorbox}
\tcbuselibrary{skins,breakable}
\usepackage{xparse}

\usepackage{graphicx}
\usepackage{subcaption}
\usepackage{booktabs}
\usepackage{multirow}
\usepackage{float}
\usepackage{wrapfig}
\usepackage{tabularx}
\usepackage{siunitx}

\usepackage{amsmath}
\usepackage{amssymb}
\usepackage{mathtools}
\usepackage{bm}
\usepackage{bbm}

\usepackage{microtype}
\usepackage{xspace}
\usepackage{xcolor}
\usepackage{color}
\usepackage{url}

\usepackage{listings}

\usepackage[numbers]{natbib}

\usepackage{hyperref}
\usepackage[capitalize,noabbrev]{cleveref}

\usepackage{orcidlink}

\usepackage[accsupp]{axessibility}

\hypersetup{
  colorlinks=true,
  linkcolor=black,
  citecolor=black,
  urlcolor=black
}

\usepackage{mdframed}

\mdfdefinestyle{TopRuleGray}{
  skipabove=2pt,
  skipbelow=2pt,
  linewidth=0pt,
  topline=true,
  bottomline=false,
  leftline=false,
  rightline=false,
  linecolor=black!40,
  backgroundcolor=black!3,
  innerleftmargin=8pt,
  innerrightmargin=8pt,
  innertopmargin=8pt,
  innerbottommargin=8pt,
  roundcorner=2pt
}

\newmdenv[style=TopRuleGray]{AssumpBox}
\newmdenv[style=TopRuleGray,backgroundcolor=white]{TheoremBox}

\newtcolorbox{takeawaysbox}[1]{%
  enhanced,
  breakable,
  colframe=black,
  colback=black!5,
  boxrule=0.9pt,
  arc=3.5pt,
  left=6pt,
  right=6pt,
  top=8pt,
  bottom=6pt,
  title={#1},
  fonttitle=\bfseries,
  coltitle=white,
  colbacktitle=black,
  boxed title style={
    size=small,
    arc=3pt,
    left=6pt,
    right=6pt,
    top=2pt,
    bottom=2pt
  },
  attach boxed title to top left={xshift=8pt,yshift=-2mm},
}

\newcounter{boxassumption}

\newenvironment{assumptionbox}[1][]{%
  \refstepcounter{boxassumption}%
  \ifstrempty{#1}{%
    \takeawaysbox{Assumption~\theboxassumption}%
  }{%
    \takeawaysbox{Assumption~\theboxassumption~(#1)}%
  }%
}{%
  \endtakeawaysbox%
}

\newcounter{boxtheorem}

\newenvironment{theorembox}[1][]{%
  \refstepcounter{boxtheorem}%
  \ifstrempty{#1}{%
    \takeawaysbox{Theorem~\theboxtheorem}%
  }{%
    \takeawaysbox{Theorem~\theboxtheorem~(#1)}%
  }%
}{%
  \endtakeawaysbox%
}

\setlist[enumerate,1]{
  label=(A\arabic*),
  ref=(A\arabic*),
  leftmargin=2.2em,
  labelsep=0.6em,
  itemsep=0.25em,
  topsep=0.25em,
  font=\normalfont
}

\providecommand{\1}{\mathbf{1}}

\newcommand{\pcview}[2]{%
  \includegraphics[width=0.31\linewidth,height=2.2cm,keepaspectratio,#2]{#1}%
}

\newlength{\qualpanelpadlen}
\newcommand{\qualpanelpad}[1]{\setlength{\qualpanelpadlen}{#1}}

\newcommand{\qualpanel}[9]{%
  \begin{minipage}[t]{0.49\linewidth}
    \centering
    \fbox{%
      \begin{minipage}[t]{0.95\linewidth}
        \raggedright
        \vspace{0.3em}

        \noindent{\small\textbf{#1}}\hfill{\small\textbf{#2}}\par
        \vspace{0.35em}

        {\centering
        \begin{tabular}{@{}c c c@{}}
          #3 & #4 & #5 \\
          \scriptsize Front & \scriptsize Side & \scriptsize Top
        \end{tabular}\par
        }

        \vspace{0.35em}

        {\scriptsize
        \textbf{Q:}~#6 \par
        \vspace{0.1em}
        \textbf{MPM:}~#7 \par
        \vspace{0.1em}
        \textbf{Qwen3-Omni (3-view):}~#8 \par
        \vspace{0.1em}
        \textbf{GT:}~#9 \par
        }

        \vspace{0.35em}
        \vspace{\qualpanelpadlen}%
      \end{minipage}%
    }%
  \end{minipage}%
  \global\setlength{\qualpanelpadlen}{0pt}%
}

\makeatletter
\newcommand{\beginsupplement}{%
  \setcounter{table}{0}%
  \renewcommand{\thetable}{A\arabic{table}}%
  \setcounter{figure}{0}%
  \renewcommand{\thefigure}{A\arabic{figure}}%
  \setcounter{equation}{0}%
  \renewcommand{\theequation}{A\arabic{equation}}%
  \@ifundefined{c@algorithm}{}{%
    \setcounter{algorithm}{0}%
    \renewcommand{\thealgorithm}{A\arabic{algorithm}}%
  }%
}
\makeatother

\begin{document}


\title{Multimodal LLMs under Pairwise Modalities}

\titlerunning{Multimodal LLMs under Pairwise Modalities}

\author{
Yan Li\inst{1}\textsuperscript{*} \and
Yunlong Deng\inst{1}\textsuperscript{*} \and
Yuewen Sun\inst{1,2} \and
Gongxu Luo\inst{1}\\[0.3em]
Kun Zhang\inst{1,2} \and
Guangyi Chen\inst{1,2}
}

\authorrunning{Y. Li et al.}

\institute{
Mohamed bin Zayed University of Artificial Intelligence
\and
Carnegie Mellon University
}

\maketitle


\begingroup
\renewcommand\thefootnote{*}
\footnotetext{Equal contribution.}
\endgroup

\begingroup
\renewcommand\thefootnote{}
\footnotetext{\hspace{-1em} {Preprint}}
\endgroup


\begin{abstract}

Despite the impressive results achieved by multimodal large language models (MLLMs), their training typically relies on jointly curated multimodal data, requiring substantial human effort to construct multi-way aligned datasets and thereby limiting scalability across domains. 
In this work, we explore training MLLMs by only leveraging multiple paired modalities as a surrogate for the full joint multimodal distribution.
Specifically, we first provide a theoretical analysis of the conditions under which the representations are identifiable with only observing pairwise modalities. Building on this analysis, we propose a representation learning framework for aligning latent representations across modalities using only pairwise data. The framework
consists of two stages: latent representation alignment and cross-modal recomposition. Specifically, in the first stage, we learn the shared latent space across modalities by both self-modal reconstruction and pair-wise contrastive learning. We also incorporate an inductive bias in the contrastive learning process by partially aligning and minimal latent specification. In stage two, we integrate the encoder of newly introduced modalities with the decoders of the pre-trained modalities to facilitate cross-modal transfer and generation.
We evaluate our method by newly adding 3D point clouds and tactile modalities into pre-trained MLLMs with three modality pairs and show that, by learning an aligned latent representation space, our model achieves strong cross-modal performance.
\end{abstract}


\section{Introduction}

Multimodal large language models (MLLMs) have recently demonstrated remarkable capabilities in cross-modal understanding and generation, enabling applications such as image-grounded dialogue~\cite{llava,llava-ov,emu3,srf}, audio-visual reasoning~\cite{xu2025qwen2,xu2025qwen3omni}, cross-modal interactions~\cite{yang2022lavt,carion2025sam}, and embodied perception~\cite{rt2,kim2024openvla}. By jointly modeling multiple modalities, e.g., text, images, audio, and 3D signals, these MLLMs learn shared representations that support flexible modality transfer and cross-modal generation. However, their success has largely relied on large-scale {jointly curated} multimodal datasets, in which all modalities are aligned for each instance~\cite{emu3,xu2025qwen3omni}. Constructing such multi-way aligned data is expensive, labor-intensive, and often infeasible in emerging domains such as robotics, tactile sensing, and medical imaging. This heavy dependence on fully aligned multimodal corpora significantly limits scalability to more modalities. 

As shown in \cref{fig:introduction}, compared with fully jointly aligned multimodal data, pairwise modality alignments are far more readily available in real-world data sources, making them a more practical supervision signal for scalable multimodal learning. For example, Image–text pairs are abundant on the web, 3D–image pairs can be obtained from multi-view capture systems, and tactile–vision pairs can be collected using instrumented robotic platforms equipped with touch sensors. This observation raises a fundamental yet unresolved question:
\begin{center}
	{ \textit{How can we train a multimodal model using only paired modalities,} } \\
	{ \textit{without ever observing the full joint distribution?} }
\end{center}
In this work, we explore multimodal representation learning under \emph{pairwise modalities} as a surrogate for full multi-way supervision. We begin by theoretically analyzing the conditions under which latent representations are identifiable from a collection of overlapping modality pairs in the generative model setting. Under mild identifiability assumptions, we show that learning shared representations from pairwise modalities is feasible and admits theoretical guaranties. Intuitively, if the modality graph formed by pairwise datasets is sufficiently connected, shared latent factors can be consistently recovered even without explicit multi-way alignment.

\begin{figure}[t]
    \centering
    \includegraphics[width=0.95\linewidth]{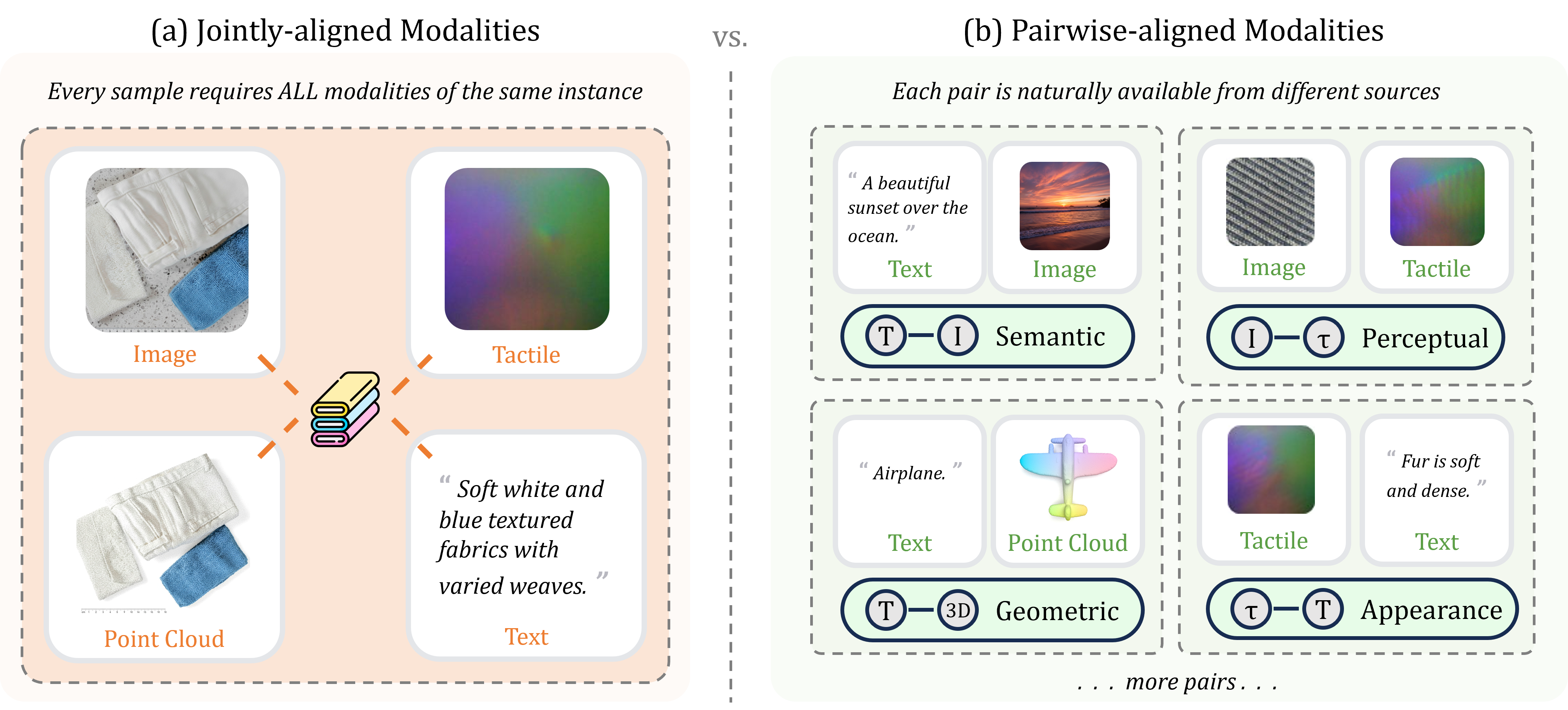}
    \caption{\textbf{Comparison between jointly-aligned and pairwise-aligned multimodal data}. 
(a) Jointly-aligned modalities: each sample requires all modalities, which is expensive and difficult to scale. 
(b) Pairwise-aligned modalities: modality pairs (e.g., text–image, image–tactile, text–3D) are naturally available from different sources, forming a connected modality graph without requiring fully aligned multi-way data.}
    \label{fig:introduction}
    \vspace{-15pt}
\end{figure}

Motivated by this analysis, we introduce a representation learning framework for aligning latent representations across modalities using only pairwise supervision. The framework is structured into two stages: first, learning a unified shared latent space; and second, facilitating cross-modality transfer via representation recomposition. Specifically, in stage one, shared representations are learned by combining self-modal reconstruction with cross-modal latent alignment. The former preserves modality-specific semantics and stabilizes training, whereas the latter employs a contrastive objective to encourage latent factors from paired modalities to align within a unified representation space. Within the contrastive learning framework, we incorporate structural inductive biases derived from the underlying data-generating process by enforcing partial representation alignment and promoting minimal latent specification. Then, in the second stage, we perform cross-modal recomposition to integrate newly introduced modalities into pre-trained large language models (LLMs) or MLLMs. For example, to incorporate a new modality such as 3D point clouds into an LLM, we first extract the shared representations between 3D and text with a point cloud encoder aligned with the textual latent space. We then recombine the aligned encoder with the text decoder, enabling text generation conditioned on 3D inputs. This recompositional design allows the new modality to leverage the generative capacity of pre-trained language models without full joint retraining. Unlike existing cross-modal adapters that require instruction tuning~\cite{llava} or LoRA-based fine-tuning~\cite{lora} to modify pre-trained models, our approach keeps the pre-trained models unchanged, thereby avoiding catastrophic forgetting and preserving their original capabilities.

We validate our approach by incorporating two new modalities, such as 3D point clouds and tactile sensing, into a pre-trained MLLM. Specifically, we adopt Qwen3-Omni~\cite{xu2025qwen3omni}, which is pre-trained on image, text, video, and audio modalities, as the base MLLM. To train and evaluate the proposed framework, we leverage publicly available datasets Objaverse~\cite{objaverse} and TVL~\cite{tvl}, comprising three new modality pairs: text--3D, text--tactile, and image--tactile. We compare our method against previous state-of-the-art MLLMs as well as adapter-based baselines. The results demonstrate that, when properly structured, pairwise modality supervision provides a scalable and effective alternative to conventional fully jointly aligned multimodal training.

In summary, the main contributions of this paper are threefold:
\begin{itemize}[topsep=2pt,itemsep=2pt,parsep=0pt,partopsep=0pt]
    \item We provide a theoretical analysis on the identifiability of multimodal representation from pairwise modality data.
    \item We propose a unified latent alignment and recomposition framework for scalable multimodal learning using only pairwise supervision.
    \item We conduct experiments with two newly introduced modalities and demonstrate strong cross-modal transfer despite relying solely on limited modality-pair datasets.
\end{itemize}

\section{Related Work}
\label{sec:related_work}

\noindent\textbf{Multimodal Representation Learning.}
Foundation scale pretraining made representation alignment a central design pattern. Dual encoder contrastive learning, exemplified by CLIP~\cite{radford2021learning} and ALIGN~\cite{jia2021scaling}, aligns image and text embeddings with large scale paired supervision, enabling strong zero shot transfer and retrieval via similarity in a shared space. Later refinements improved objectives and scaling strategies, including locked image text tuning, sigmoid based contrastive losses, and hybrid encoder decoder formulations that unify contrastive alignment with captioning~\cite{zhai2022lit,zhai2023sigmoid,yu2022coca}. Complementary to dual encoders, fusion based vision language pretraining integrates modalities inside a transformer using cross attention, with representative instances such as ViLBERT, LXMERT, UNITER, ViLT, and ALBEF~\cite{lu2019vilbert,tan2019lxmert,chen2020uniter,kim2021vilt,li2021align}. Unified task formulations further connect representation learning with generation; for example, OFA, FLAVA, BLIP, and BLIP 2~\cite{wang2022ofa,singh2022flava,li2022blip,li2023blip}.
Beyond image text, alignment has been extended to video and audio through video text pretraining and multimodal self-supervised learning, as well as audio text contrastive learning~\cite{xu2021videoclip,akbari2021vatt,elizalde2023clap}. Recent work aims to bind multiple modalities into a single latent space, including ImageBind and language centered binding via language as a hub, such as LanguageBind and OneLLM~\cite{girdhar2023imagebind,zhu2023languagebind,han2024onellm}. Closely related to our motivation, unpaired multimodal representation learning reduces the need for fully aligned multimodal tuples by connecting or extending existing contrastive representation spaces, as in C-MCR~\cite{wang2023connecting} and Ex-MCR~\cite{zhang2024extending}. Another relevant line is missing-modality learning, which studies inference or training when only partial modalities are observed, including multimodal VAEs, SMIL, and shared-specific feature modeling~\cite{wu2018multimodal,ma2021smil,wang2023shaspec}. Our work follows the spirit of representation alignment but emphasizes explicit latent space alignment through a chain of paired datasets, targeting improved efficiency and a principled mechanism for modality extension. While our setting can be viewed as a structured special case of missing-modality learning, it differs in that we infer joint representations from sparsely connected observations rather than assuming that a shared embedding learned from partial modalities directly transfers to missing-modality cases.

\vspace{0.1cm}
\noindent\textbf{Causal Representation Learning.}
Causality formalizes how latent mechanisms generate observations and how interventions affect outcomes, commonly described by structural causal models and related graphical criteria~\cite{pearl2009causality,spirtes2000causation,peters2017elements}. Causal representation learning aims to discover representations corresponding to high level causal variables that support robust generalization and modular recombination across domains~\cite{scholkopf2021toward}. A core difficulty is identifiability: how can the estimated latent variables be disentangled and correspond to the ground truth latent variables in block level or one-to-one mapping, especially in nonlinear cases~\cite{locatello2019challenging}. Recent progress establishes identifiability by leveraging sufficient changes~\cite{yao2022temporally,zhang2024causal}, distribution shifts~\cite{kong2022partial,li2023subspace}, multi view learning~\cite{yao2023multi,sun2024causal}, contrastive learning~\cite{zimmermann2021contrastive,hyvarinen2019nonlinear}, supervised learning~\cite{von2021self,halva2020hidden}, or explicit structural constraints like sparsity~\cite{lachapelle2022partial,zheng2022identifiability}. In language models, causal tools have been used to connect internal representations to behavior, spanning causal mediation analysis and causal abstraction. Recent surveys advocate integrating causality throughout pretraining, fine-tuning, alignment, inference, and evaluation~\cite{wu2024causality}. Our approach aligns with this agenda by using interlocking paired datasets across modalities to identify a shared latent space that remains consistent when incorporating new modalities.

\vspace{0.1cm}
\noindent\textbf{Multimodal Large Language Models.}
Multimodal large language models extend transformer-based language models with modality encoders and connectors and are typically trained via staged pretraining, alignment, and instruction tuning~\cite{vaswani2017attention,brown2020language}. Early strong results include Flamingo for few shot visual language adaptation and PaLM E for embodied multimodal reasoning~\cite{alayrac2022flamingo,driess2023palm}. Closed-source systems such as GPT 4 and Gemini demonstrate strong multimodal capabilities at scale~\cite{achiam2023gpt,team2023gemini}. Open and instruction tuned MLLMs have rapidly progressed, including Kosmos 1, InstructBLIP, LLaVA,  MiniGPT 4, and parameter-efficient visual instruction adaptation~\cite{huang2023language,dai2023instructblip,liu2023visual,zhu2023minigpt,gao2023llama,cai2024vip}. Scaling alignment with stronger vision foundations is explored in InternVL~\cite{chen2024internvl}, and any-to-any generation systems, such as Next-GPT, connect an LLM to multiple modality decoders~\cite{wu2024next}.
Within open source ecosystems, the Qwen series has released multiple multimodal technical reports, including Qwen2-VL, Qwen2-Audio, Qwen2.5-Omni, Qwen3-VL, and Qwen3-Omni~\cite{wang2024qwen2,chu2024qwen2,xu2025qwen2,xu2025qwen3omni,bai2025qwen3vl}. Most modality extension approaches rely on large-scale joint datasets or extensive modality-specific finetuning. While some works~\cite{xu2025qwen3omni,xu2021videoclip} are LanguageBind-style approaches that use language as a shared anchor, our framework does not require any universal anchor modality, enabling more flexible updates for newly emerging modality pairs. We implement explicit representation-level latent alignment on top of Qwen3 Omni using a chain of paired datasets, enabling new modalities and cross-modal generation while preserving original multimodal abilities.

\section{Theory}
\label{sec: theory}

\subsection{Generative Model and Problem Setup}
We introduce the following generative model as a formal tool for studying identifiability: it specifies how modality-specific observations arise from shared and private latent factors, thereby allowing us to state when the shared latent subspace can be recovered from pairwise observations.

\begin{wrapfigure}{r}{0.58\linewidth}
\vspace{-20pt}
\centering
\includegraphics[width=\linewidth]{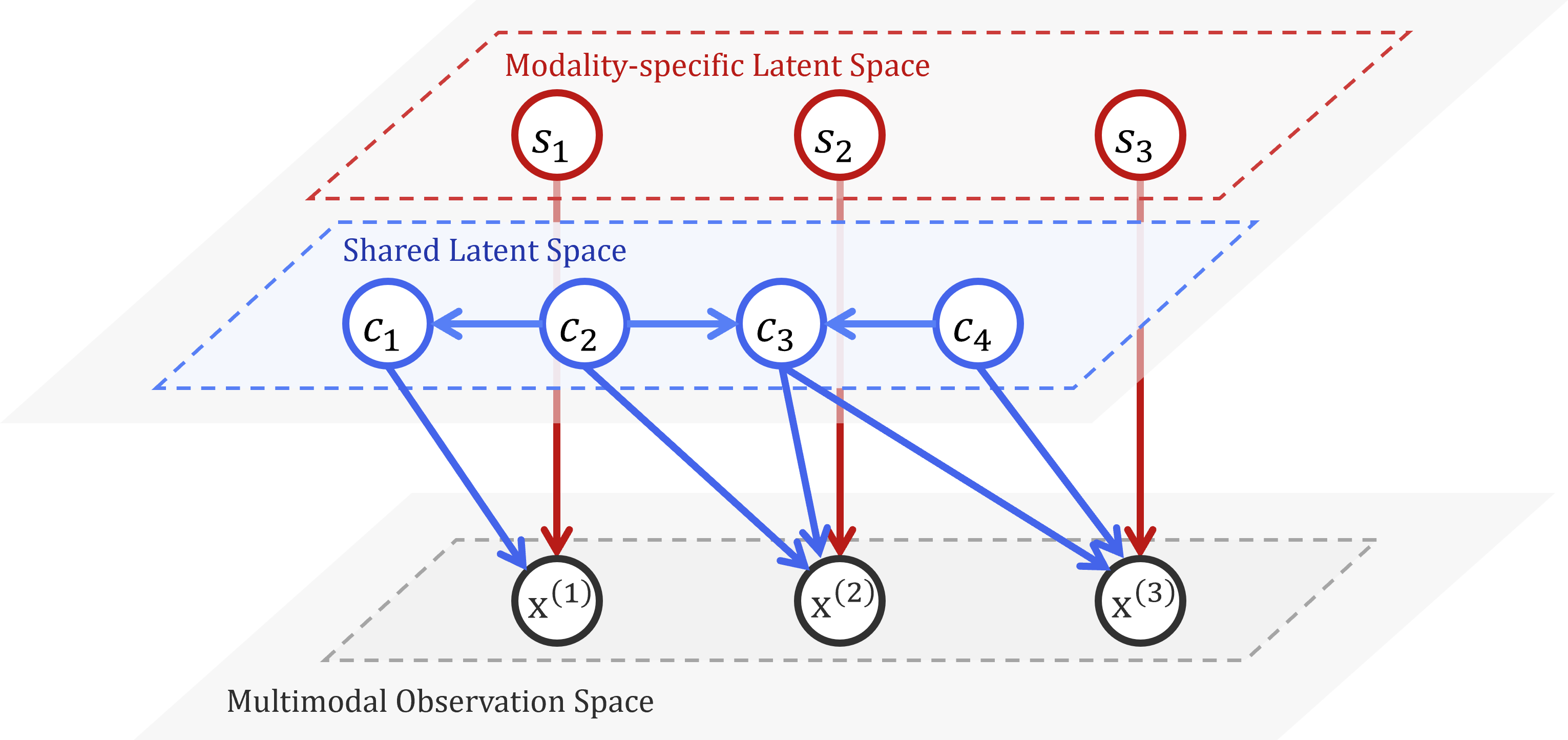}
\vspace{-8pt}
\caption{\textbf{Data generation process.} The shared latent space $\vec{c}$ (blue) contains factors that influence multiple modalities, while modality-specific variables $\vec{s}$ (red) affect only their corresponding modality. Each observation $\vec{x}^{(m)}$ is generated from its shared latent block $\vec{z}^{(m)}_c$ and modality-specific block $\vec{z}^{(m)}_s$.}
\label{fig:causal_map}
\vspace{-25pt}
\end{wrapfigure}
\subsubsection{Data Generating Process.}
We consider $\vec{x} := [\vec{x}^{(1)}, \dots, \vec{x}^{(M)}]$ to denote a collection of observations from $M$ distinct modalities, where each $\vec{x}^{(m)} \in \mathbb{R}^{d_x^{(m)}}$ represents the observation corresponding to modality $m$, and $d_x^{(m)}$ denotes its dimensionality. Let $\vec{z} := [\vec{z}^{(1)}, \dots, \vec{z}^{(M)}]$ denote the collection of latent variables associated with the $M$ modalities, which are assumed to be causally related and to jointly govern the data generating process of the observed variables. For each modality, the latent variables $\vec{z}^{(m)}= [\vec{z}^{(m)}_c, \vec{z}^{(m)}_s] \in \mathbb{R}^{d_z^{(m)}}$ consist of a {shared} latent block $\vec{z}^{(m)}_c$ and an independent {modality-specific} latent block $\vec{z}^{(m)}_s$. The shared latent variables $\vec{z}^{(m)}_c \in \mathbb{R}^{d_c^{(m)}}$ denote the subset of latent factors that are causally associated with latent variables from other modalities, or directly participate in the generation of multiple modalities. In contrast, the modality-specific latent variables $\vec{z}^{(m)}_s \in \mathbb{R}^{d_s^{(m)}}$ are exclusively related to modality $m$ and influence only its corresponding observations.

As shown in Figure~\ref{fig:causal_map}, the data generation process can be written as:
\begin{equation}
\label{eq:gen_single}
 \vec{z}^{(m)}_i = f_{m, i}(\text{Pa}(\vec{z}^{(m)}_i), \epsilon^{(m)}_i) ,\qquad \vec{x}^{(m)} = g_{m}(\vec{z}^{(m)}_c, \vec{z}^{(m)}_s),
\end{equation}
where the $f_{m, i}$ denotes the latent causal process of each latent variable, while the $g_m$ is the generative function from latent factors to the observation. $\text{Pa}(\cdot)$ denotes the parents of a variable, and $\epsilon^{(m)}_i \sim p_{\epsilon^{(m)}_i}$ is the noise term. 
Since certain shared latent factors may simultaneously appear in the generating processes of multiple modalities (e.g., ${c}_3$ in Figure~\ref{fig:causal_map}), the total number of distinct shared latent variables is generally smaller than $\sum_{m=1}^{M} |\vec{z}^{(m)}_c|$. To make it clearer, we define the set of all connected latent factors as $\vec{c}=\{c_1,c_2, \dots, c_{d_c}\}$, and we have $\vec{c}=\{\vec{z}^{(1)}_c \cup \vec{z}^{(2)}_c \cup ... \cup \vec{z}^{(M)}_c \}$. Similarly, we define the set of modality-specific latent variables as $\vec{s}=\{s_1,s_2, \dots, s_{d_s}\}$. 

\subsubsection{Pairwise Observations and Modeling.}
In practice, it is often challenging to collect fully jointly aligned multimodal data as described in Eq.~\ref{eq:gen_single}. Instead, it is considerably more feasible to obtain observations that are aligned only in pairwise modalities. 
To formally characterize such pairwise co-occurrence, we introduce an undirected graph $\mathcal{G} = ([M], \mathcal{E})$, where the node set $[M] := \{1, \dots, M\}$ indexes the modalities, and the edge set $\mathcal{E} \subseteq \bigl\{ \{i,j\} : 1 \le i < j \le M \bigr\}$
encodes the availability of pairwise observations. An edge $\{i,j\} \in \mathcal{E}$ indicates that paired data between modalities $i$ and $j$ are observed.
For a fixed target modality $i \in [M]$, we define its set of observed neighbors as $
\mathcal{N}(i) := \{\, j \in [M] \setminus \{i\} \;:\; \{i,j\} \in \mathcal{E} \,\} $
which represents the collection of modalities for which pairwise aligned data with modality $i$ are available. Accordingly, the available dataset consists of pairwise observations $
\mathcal{D} = \bigl\{ (\vec{x}^{(i)}, \vec{x}^{(j)}) \;:\; \{i,j\} \in \mathcal{E} \bigr\} $, that is, for each edge $\{i,j\} \in \mathcal{E}$ in the modality graph, we observe samples jointly drawn from modalities $i$ and $j$.

Given one pairwise observations $(\vec{x}^{(i)}, \vec{x}^{(j)} )$, and taking modality $i$ as target, we assume that the paired data between modality $i$ and any neighbor $j \in \mathcal{N}(i)$ admits the following form
\begin{equation}
\label{eq:pair_gen}
\vec{x}^{(i)} = g_i\bigl(\vec{z}^{(i)}_c, \vec{z}^{(i)}_s\bigr),
\qquad
\vec{x}^{(j)} = g_{j\leftarrow i}\bigl(\vec{z}^{(i)}_c, \tilde{\vec{z}}^{(j)}_{\setminus i}\bigr),
\end{equation}
where the former denotes the standard generating process, $\vec{z}^{(i)}_c$ denotes all latent variables of modality $i$ that are relevant to other modalities, and $\tilde{\vec{z}}^{(j)}_{\setminus i}$ represents the latent factors of modality $j$ that are needed to generate $\vec{x}^{(j)}$ expect $\vec{z}^{(i)}_c$. As illustrated in Figure~\ref{fig:causal_map}, when $i = 2$ and $j = 3$, the shared latent variables are given by $\vec{z}^{(i)}_c = \{c_2, c_3\}$, while $\tilde{\vec{z}}^{(j)}_{\setminus i}$ corresponds to the set $\{c_4, s_3\}$, which includes the latent factors of modality $j=3$ that are not shared with modality $i=2$.

\subsubsection{Identifiability.} 
Given the data $\mathcal{D}$, our goal is to learn latent variables $\vec{z}^{(m)}$ underlying each modality. Formally, consider the ground truth 
data-generating process in ~\cref{eq:gen_single}, whose parameters are $\Bigl\{ g_m,\; \{ f_{m,i} \}_{i=1}^{d_z^{(m)}},\; \{ p(\epsilon^{(m)}_i) \}_{i=1}^{d_z^{(m)}} \Bigr\}_{m=1}^{M}$, and the estimated parameters $\Bigl\{ \hat{g}_m,\; \{ \hat{f}_{m,i} \}_{i=1}^{d_z^{(m)}},\; \{ \hat{p}(\epsilon^{(m)}_i) \}_{i=1}^{d_z^{(m)}} \Bigr\}_{m=1}^{M}$. Our subspace identification objective is to show that, with the same data, the estimated shared latent representations $\hat{z}_{c}^{(m)}$ based on estimated parameters should be equivalent to the ground truth counterpart ${z}_c^{(m)}$
as $\hat{\vec{z}}^{(m)}_c = h^{(m)} \vec{z}^{(m)}_c$, where $h^{(m)}$ is an invertible mapping. 

This notion of \emph{subspace identifiability} guarantees that, although individual latent components may not be uniquely identifiable, the shared latent subspace responsible for cross-modal dependence is recoverable from pairwise observations. Such identifiability is sufficient to guarantee consistent cross-modal alignment and reliable transfer across modalities, thereby providing a principled foundation for real-world multimodal learning applications.

\subsubsection{Example (Pairwise Multimodal Setting).}
Consider three modalities in Fig.~\ref{fig:causal_map}: image ($m=1$), tactile sensing  ($m=2$), and text ($m=3$), with observations $\vec{x}^{(1)}, \vec{x}^{(2)}, \vec{x}^{(3)}$. The latent variables $\vec{z}=[z^{(1)}, z^{(2)}, z^{(3)}] = [[c_1, s_1], [c_2, c_3, s_2],
[c_3, \\c_4,s_3]]$ capture underlying semantic factors, where the shared block $c = \{c_1, c_2, c_3, \\c_4\}$ represents object-level properties that influence multiple modalities, while the modality-specific block $s = \{s_1, s_2, s_3\}$ accounts for modality-dependent variations such as sensor noise. In practice, we may only observe pairwise data, such as image–text, text–tactile, or image–tactile pairs, without fully aligned triplets. Nevertheless, because the shared latent factors propagate across the connected modality graph, the common semantic subspace can still be recovered from overlapping pairwise observations.

\subsection{Block Identifiability}

\subsubsection{Notation.} 
Here, we introduce several notations that will be used in the subsequent theoretical analysis. 
For a vector $\mathbf{v} \in \mathbb{R}^D$, we denote by $[\mathbf{v}]_k$ its $k$-th coordinate and by $[\mathbf{v}]_I$ the subvector indexed by $I \subseteq [D] := \{1,\dots,D\}$. 
For a vector-valued function $\mathbf{f}$, $[\mathbf{f}]_k$ denotes its $k$-th output component. 
Let $J_f$ denote the Jacobian matrix of a function $f$ with respect to its concatenated input variables. 
We use $I(\vec{z}^{(m)}_c)$ to denote the set of coordinate indices corresponding to the shared latent block $\vec{z}^{(m)}_c$ within the full latent vector.
Accordingly, the partial Jacobian of $g_{j\leftarrow i}\bigl(\vec{z}^{(i)}_c, \tilde{\vec{z}}^{(j)}_{\setminus i}\bigr)$ (cf.~\cref{eq:pair_gen}), which captures the sensitivity of $\vec{x}^{(j)}$ with respect to $\vec{z}^{(i)}_c$, can be expressed as
\[
A_{j\leftarrow i}(\vec{z}^{(i)}_c, \tilde{\vec{z}}^{(j)}_{\setminus i})
:=
\bigl[J_{g_{j\leftarrow i}\bigl(\vec{z}^{(i)}_c, \tilde{\vec{z}}^{(j)}_{\setminus i}\bigr)}\bigr]_{:,\,I(\vec{z}^{(m)}_c)}
\in\mathbb{R}^{d_{x}^{(j)}\times d_{c}^{(i)}}.
\]

\begin{assumptionbox}[Subspace Identifiability Conditions]
\label{subspace_ass}
\begin{enumerate}
\item \label{subspace_ass:A1} \textbf{Smoothness \& Invertibility:}
The generative function of single modality $g_i$ and the joint one $g_{ij}:(\vec{z}^{(i)}_c, \vec{z}^{(i)}_s, \tilde{\vec{z}}^{(j)}_{\setminus i}) \rightarrow (\vec{x}^{(i)},\vec{x}^{(j)})$ are smooth and admit smooth inverse functions.
\item \label{subspace_ass:A2} \textbf{Linear Independence:}
For latent realization $(\vec{z}^{(i)}_c, \tilde{\vec{z}}^{(j)}_{\setminus i}), j \in \mathcal{N}(i)$, there exist matrices
\(
L_{ij}(\vec{z}^{(i)}_c, \tilde{\vec{z}}^{(j)}_{\setminus i})\in\mathbb{R}^{d^{(i)}_{c}\times d^{(j)}_{x}},
\) such that the following coefficient combination identity holds:
\[
\sum_{j\in\mathcal{N}(i)} L_{ij}(\vec{z}^{(i)}_c, \tilde{\vec{z}}^{(j)}_{\setminus i})\;A_{j\leftarrow i}(\vec{z}^{(i)}_c, \tilde{\vec{z}}^{(j)}_{\setminus i})
=
I(\vec{z}^{(i)}_c).
\]

\end{enumerate}
\end{assumptionbox}
\subsubsection{Remark 1.} Assumption~\ref{subspace_ass}\ref{subspace_ass:A1} means that each observation contains enough information about its underlying latent factors so that one can, in principle, recover a latent representation from the data, with only a change of coordinates ambiguity.This premise is common in nonlinear latent variable identifiability~\cite{kong2022partial,yao2023multi,sun2024causal,zimmermann2021contrastive}.
Assumption~\ref{subspace_ass}\ref{subspace_ass:A2} captures the idea that the shared factors influence the neighboring modalities in a sufficiently rich way so that those neighbors provide informative constraints for recovering the shared part. Importantly, it does not require any observed pair to be fully informative; one neighbor may reveal only certain directions of the shared factors, while another may reveal complementary directions. Instead, it requires that, when information is aggregated across all observed neighbors, the resulting constraints span all directions of the shared latent block.

\begin{theorembox}[Subspace Identifiability with Pairwise Modalities]
\label{thm:multi_pair_ident}
With pairwise observations, a fixed target modality $i$ and its observed neighbors $\mathcal{N}(i)$, if the parameters $\{ g_i,\ \{g_{j\leftarrow i}\}_{j\in\mathcal{N}(i)}\}$ of the pairwise data-generating processes for $(\vec{z}^{(i)}_c, \vec{z}^{(i)}_s)$ satisfy Assumption~\ref{subspace_ass}\ref{subspace_ass:A1} and \ref{subspace_ass:A2}, the estimated ones $\{\, \hat g_i,\ \{\hat g_{j\leftarrow i}\}_{j\in\mathcal{N}(i)}\}$ for $(\hat{\vec{z}}^{(i)}_c,\hat{\vec{z}}^{(i)}_s)$  satisfy Assumption~\ref{subspace_ass}\ref{subspace_ass:A1}, and generative models induce identical pairwise observations for all $j\in\mathcal{N}(i)$,
then the shared latent block $z_c^{(i)}$ can be identified.
\end{theorembox}

\subsubsection{Proof Intuition.}
We compare two specifications that induce the same pairwise distributions for all observed neighbors. By Assumption~\ref{subspace_ass}\ref{subspace_ass:A1}, each observation can serve as a local coordinate chart for its latent variables, so we can track how small changes in the shared coordinates propagate to each observation of neighbors. Since the pairwise distributions match, these local input-output change patterns must also match up to a smooth change of coordinates. The key step is then to aggregate the neighbor-wise constraints using Assumption~\ref{subspace_ass}\ref{subspace_ass:A2}: combining the constraints across multiple neighbors forces the recovered shared representation to be related to the true shared block by a smooth invertible transformation, establishing block identifiability even when no single pair is sufficient on its own. A detailed proof of block identifiability, along with additional results on component identifiability, is provided in Appendix A.

\section{Approach}

In this section, we present a framework, {M}ultimodal {P}airwise {M}odel( MPM in short), for multimodal representation learning using only pairwise modality supervision. This framework is motivated by our theoretical insight that, under appropriate connectivity, shared latent representations can be recovered from overlapping modality pairs without ever observing fully jointly aligned samples. 

\subsection{Overall Framework}

Formally, as shown in the theoretical setup, suppose that we are given a set of modalities $\mathcal{M} = \{1, \dots, M\}$. Instead of observing samples from the full joint distribution, we only have access to datasets of modality pairs
\begin{equation}
\label{eq:pair_dataset}
\mathcal{D}_{ij} = \{(x^{(i)}_n, x^{(j)}_n)\}_{n=1}^{N_{ij}},
\qquad (i,j) \in \mathcal{E},
\end{equation}
where $\mathcal{E}$ denotes the set of observed modality pairs as defined in \cref{sec: theory}, and $N_{ij}$ denotes the number of data pairs in modality pair $\vec{x}^{(i)},\vec{x}^{(j)} $.
Our goal is to learn the latent representation $\vec{z}^{(m)}= [\vec{z}^{(m)}_c, \vec{z}^{(m)}_s]$ of each modality $\vec{x}^{(m)}$, and most importantly, to find the shared latent variables $\vec{z}^{(m)}_c$, which are semantically aligned across modalities, enabling cross-modal transfer and generation.

As shown in Figure~\ref{fig:method}, our framework consists of two stages:
\begin{enumerate}[label={},topsep=2pt,itemsep=2pt,parsep=0pt,partopsep=0pt]
    \item \textbf{Stage I: Latent Representation Alignment.} Learn modality encoders and decoders that map inputs into a unified shared latent space using only pairwise supervision.
    \item \textbf{Stage II: Cross-Modal Recomposition.} Recompose the aligned encoders of newly introduced modalities with the decoders of existing modalities in pre-trained LLMs or MLLMs, thereby enabling cross-modal generation without modifying the backbone model.
\end{enumerate}

\begin{figure}[t]
\centering
\includegraphics[width=0.95\linewidth]{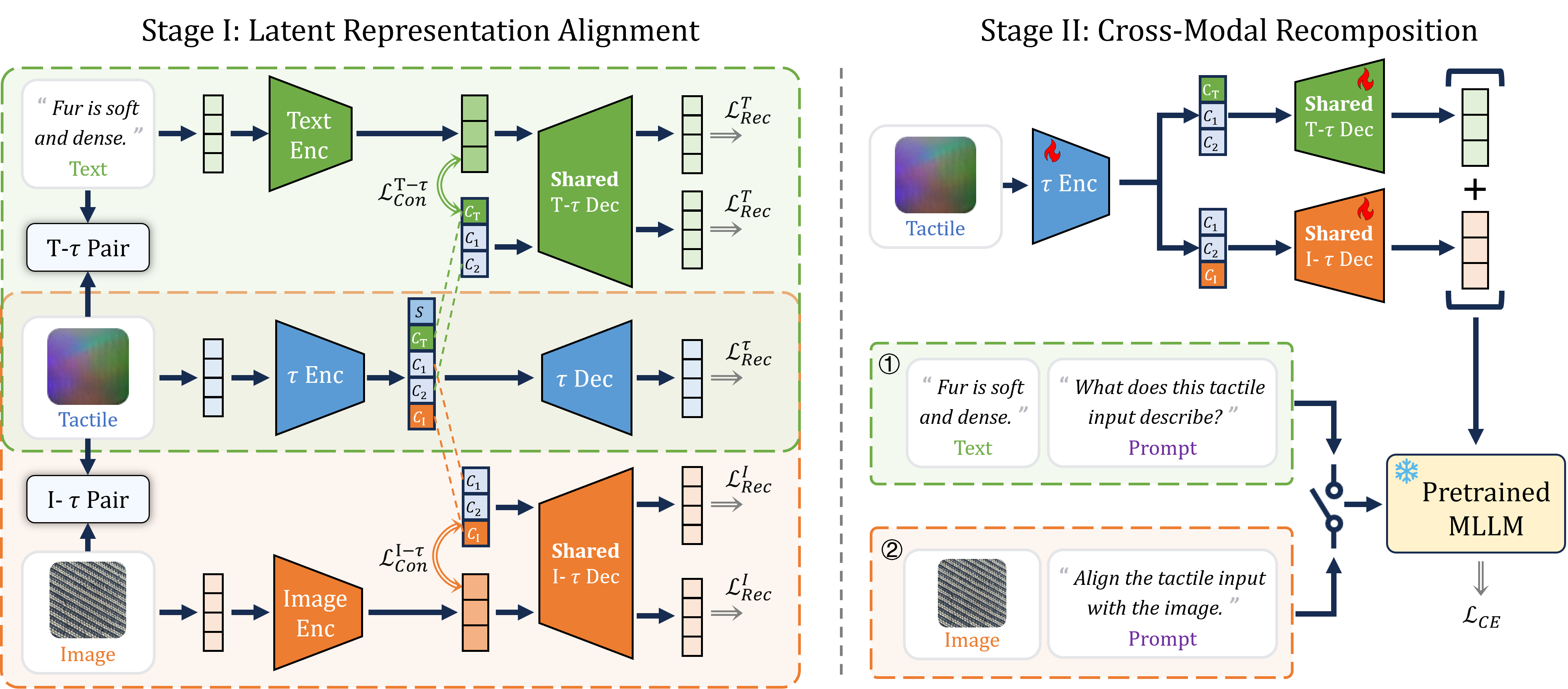}
\caption{\textbf{Two-stage framework for multimodal learning from pairwise-only supervision.}
Stage I (Left) aligns modality encoders in a shared latent space; and Stage II (Right) recomposes them with a frozen pretrained MLLM to enable cross-modal generation without multi-way alignment.}
\label{fig:method}
\vspace{-0.4cm}
\end{figure}

\subsection{Stage I: Latent Representation Alignment}

\subsubsection{Self-modal Reconstruction.}
We begin our approach with a generative framework that reconstructs each modality from its learned latent representation, implemented through an autoencoder architecture.
For each modality $\vec{x}^{(m)}$, we first learn an encoder, such as a Transformer~\cite{vaswani2017attention} or a lightweight multilayer perceptron (MLP), that maps the observed data into latent space: 
\begin{equation}
\label{eq:encoder_def}
\text{Enc}_m:\mathbb{R}^{d_x^{(m)}} \rightarrow 
\mathbb{R}^{d_c^{(m)}} \times \mathbb{R}^{d_s^{(m)}},
\qquad
(\hat{\vec{z}}^{(m)}_c,\hat{\vec{z}}^{(m)}_s)
=
\text{Enc}_m(\vec{x}^{(m)}).
\end{equation}
which mirrors the theoretical decomposition $\vec{z}^{(m)}=[\vec{z}^{(m)}_c,\vec{z}^{(m)}_s]$.
The shared code $\hat{\vec{z}}^{(m)}_c$ is intended to parameterize the modality-relevant subset, while $\hat{\vec{z}}^{(m)}_s$ captures modality-specific variations.

To ensure that $(\hat{\vec{z}}^{(m)}_c,\hat{\vec{z}}^{(m)}_s)$ preserves all information about $\vec{x}^{(m)}$, we introduce a lightweight decoder: 
\begin{equation}
\label{eq:decoder_def}
\text{Dec}_m:\mathbb{R}^{d_c^{(m)}} \times \mathbb{R}^{d_s^{(m)}} 
\rightarrow \mathbb{R}^{d_x^{(m)}},
\qquad
\hat{\vec{x}}^{(m)} 
= 
\text{Dec}_m\bigl(\hat{\vec{z}}^{(m)}_c,\hat{\vec{z}}^{(m)}_s\bigr).
\end{equation}
Here, all latent variables are used to generate the observation. Then, we minimize
\begin{equation}
\label{eq:rec_loss}
\mathcal{L}^{(m)}_{\mathrm{rec}}
=
\mathbb{E}_{\vec{x}^{(m)}}\!\left[
\ell_m\!\left(
\hat{\vec{x}}^{(m)},\, \vec{x}^{(m)}
\right)
\right],
\end{equation}
where $\ell_m$ denotes a modality-specific reconstruction loss that measures the discrepancy between the reconstruction $\hat{\vec{x}}^{(m)}$ and the original input $\vec{x}^{(m)}$. In our implementation, we adopt a weighted combination of the mean squared error (MSE) and a cosine similarity term:
\begin{equation}
\label{eq:recon_loss_full}
\ell_m\bigl(\hat{\vec{x}}^{(m)}, \vec{x}^{(m)}\bigr)
=
\left\lVert \hat{\vec{x}}^{(m)} - \vec{x}^{(m)} \right\rVert_2^2
\;-\;
\lambda \,
\hat{\vec{x}}^{(m)\top}\vec{x}^{(m)}
/
(
\lVert \hat{\vec{x}}^{(m)}\rVert_2
\lVert \vec{x}^{(m)} \rVert_2
).
\end{equation}
where $\lambda$ controls the trade-off between the MSE and the cosine similarity.

\subsubsection{Cross-Modal Contrastive Alignment.}
To align the shared latent representations across modalities, we adopt a CLIP-style contrastive learning objective. Following the symmetric formulation of CLIP, we enforce bidirectional alignment between modalities $i$ and $j$. Given the shared representations $\hat{\vec{z}}^{(i)}_c$ and $\hat{\vec{z}}^{(j)}_c$, the standard contrastive loss from $i \rightarrow j$ is defined as
\begin{equation}
\mathcal{L}_{i \rightarrow j}
=-
\mathbb{E}
\left[
\log
\frac{
\exp\!\left(
\mathrm{sim}(\hat{\vec{z}}^{(i)}_c, \hat{\vec{z}}^{(j)}_c) / \tau
\right)
}{
\sum_{(\vec{x}^{(i)}, \vec{x}^{(j')})}
\exp\!\left(
\mathrm{sim}( \hat{\vec{z}}_c^{(i)}, \hat{\vec{z}}_c^{(j')}) / \tau
\right)
}
\right],
\end{equation}
where $\mathrm{sim}(\cdot,\cdot)$ denotes cosine similarity and $\tau$ is a temperature parameter. $(\vec{x}^{(i)},\vec{x}^{(j')})$ are negative pairs. The reverse direction $j \rightarrow i$ is defined analogously. 
The final alignment loss for the pair $\{i,j\}$ is
\begin{equation}
\label{eq:align_loss}
\mathcal{L}^{(ij)}_{\mathrm{con}}
=
\frac{1}{2}
\left(
\mathcal{L}_{i \rightarrow j}
+
\mathcal{L}_{j \rightarrow i}
\right).
\end{equation}
We combine self-modal reconstruction and cross-modal contrastive alignment:
\begin{equation}
\label{eq:stage1_loss}
\mathcal{L}^{(ij)}_{\mathrm{Stage\,I}}
=
\lambda_{\mathrm{rec}}
\left(
\mathcal{L}^{(i)}_{\mathrm{rec}}
+
\mathcal{L}^{(j)}_{\mathrm{rec}}
\right)
+
\lambda_{\mathrm{con}}
\mathcal{L}^{(ij)}_{\mathrm{con}},
\end{equation}
where $\lambda_{\mathrm{rec}}$ and $\lambda_{\mathrm{con}}$ balance reconstruction and alignment, respectively.

\subsubsection{Partial and Asymmetric Alignment.}

According to the pairwise generative form in~\cref{eq:pair_gen}, when treating modality $i$ as the target modality, the cross-modal dependence between $\vec{x}^{(i)}$ and $\vec{x}^{(j)}$ is mediated only by the subset of shared latent variables that are relevant to modality $i$. In other words, not all components of $\vec{z}^{(j)}_c$ necessarily contribute to the generation of $\vec{x}^{(i)}$. This observation is particularly important in our setting, where the ultimate objective is to integrate newly introduced modalities into pre-trained LLMs/MLLMs. We therefore treat modalities that already exist in the pre-training stage as {target modalities}. Due to modality gaps, certain latent factors in a newly added modality may not correspond to those in the target modality. For example, tactile sensing may capture fine-grained surface properties (e.g., roughness, stiffness, or friction) that are not explicitly described in textual captions, even though some of these properties are visually correlated with images. Enforcing full alignment between $\vec{z}^{(i)}_c$ and $\vec{z}^{(j)}_c$ may therefore introduce spurious constraints.

To address this issue, we propose an asymmetric and partial alignment mechanism from the newly added modality to the target modality. Specifically, we introduce a binary mask $\vec{m}^{(j\leftarrow i)} \in \{0,1\}^{d_c^{(j)}}$ to select the subset of shared latent components in $\vec{z}^{(j)}_c$ that are relevant to modality $i$. The similarity used in the contrastive objective is then defined as
\begin{equation}
\label{eq:masked_similarity}
\mathrm{sim}\bigl(\vec{z}_c^{(i)}, \vec{z}_c^{(j)}\bigr)
\;:=\;
\mathrm{sim}\bigl(
\vec{z}_c^{(i)},\;
\vec{m}^{(j\leftarrow i)} \odot \vec{z}_c^{(j)}
\bigr),
\end{equation}
where $\odot$ denotes element-wise multiplication. 
Thus, we feed the selected shared representation of the newly added modality into the decoder of the target modality,
$
\tilde{\vec{x}}^{(i)}
=
\mathrm{Dec}_{i}\!\left(
\vec{m}^{(j\leftarrow i)} \odot \vec{z}_c^{(j)}
\right)$
and impose a reconstruction constraint on $\tilde{\vec{x}}^{(i)}$. In our experiments, we assume that the representations learned by the pre-trained MLLMs have already captured all latent variables, including shared factors and specific factors. Therefore, we focus solely on modeling and aligning the shared representations when integrating newly introduced modalities. As illustrated in \cref{fig:method}, for simplicity, we predefine the number of active entries in the mask, and the mask $\vec{m}^{(j\leftarrow i)}$ is allowed to vary across different target modalities to reflect modality-dependent shared subspaces.

\subsection{Stage II: Cross-Modal Recomposition}

After learning modality-specific encoders $\{\text{Enc}_m\}_{m=1}^{M}$ and decoders $\{\text{Dec}_m\}_{m=1}^{M}$ in Stage I, we enable cross-modal generation by recomposing the learned shared representations with a frozen pre-trained LLM/MLLM. The key idea is to use the shared latent space as a bridge between newly introduced modalities and the modalities already supported by the pre-trained backbone.

Specifically, let modality $(i)$ denote a target modality involved in the pre-training of LLMs/MLLMs, and let modality $(j)$ be newly introduced. Given an input $\vec{x}^{(j)}$, we first obtain its shared representation via $\text{Enc}_j(\vec{x}^{(j)})$. We then apply the asymmetric mask $\vec{m}^{(j\leftarrow i)}$ to select the components relevant to modality $(i)$, and  feed the resulting representation into the decoder of modality $(i)$:
\begin{equation}
\label{eq:cross_modal_transfer}
\tilde{\vec{x}}^{(i)}
=
\mathrm{Dec}_i\!\left(
\vec{m}^{(j\leftarrow i)} 
\odot 
\mathrm{Enc}_j\!\left(\vec{x}^{(j)}\right)
\right).
\end{equation}
The transferred representation $\tilde{\vec{x}}^{(i)}$ is treated as a surrogate input in modality $(i)$ and is subsequently provided as context to the frozen pre-trained LLM/MLLM for downstream generation. 
During this stage, the parameters of the backbone model remain fixed, while only the modality encoders and decoders are updated. 
This design preserves the original capabilities of the pre-trained model and avoids catastrophic forgetting.

\vspace{0.2cm}
\noindent\textbf{One Modality with Multiple Pairs.}
In practice, a modality may participate in multiple pairwise datasets. 
To accommodate this setting, we jointly transfer the modality to all its corresponding target modalities and aggregate the resulting outputs as contextual inputs to the pre-trained MLLM.
As illustrated in Figure~\ref{fig:method}, since fully jointly aligned multimodal data are unavailable, training proceeds iteratively over individual modality pairs. 
In each iteration, we sample one observed pair and update the corresponding alignment components, while keeping all decoders active. 
Although a specific modality may not appear in the current mini-batch, its decoder remains trainable due to the shared latent space established in Stage I, which ensures global consistency across modalities.





\section{Experiments}

\begin{table}[t]
\caption{\textbf{Performance comparison with baselines on two 3D point benchmarks: 3D MM-VET and ModelNet40.}
``Tuning'' indicates whether the model is fully fine-tuned (Full) or adapted via parameter-efficient fine-tuning (PEFT); when available, we report the number of updated parameters in parentheses. Higher is better. All scores are accuracy (\%).
\textsuperscript{\textdagger} Results marked with a dagger are taken from the original PointLLM and ShapeLLM papers.}
\label{tab:mpm_3d_benchmarks}
\centering
\small
\setlength{\tabcolsep}{4pt}
\renewcommand{\arraystretch}{0.88}

\sisetup{
  detect-weight = true,
  detect-inline-weight = math,
  table-number-alignment = center
}

\newcolumntype{Y}{>{\centering\arraybackslash}X}

\begin{tabularx}{\linewidth}{@{} Y c Y S[table-format=2.1] S[table-format=2.1] @{}}
\toprule
\textbf{Method} & \textbf{Tuning} & \textbf{Input} & {\textbf{3D MM-VET}} & {\textbf{ModelNet40}} \\
\midrule
Qwen3-Omni-30B-A3B & -- & 3-view 2D image & 56.1 & 24.2 \\
\midrule
PointBind\&LLM\textsuperscript{\textdagger}  & LoRA($\sim$7B)         & 3D point cloud & 23.5 & 51.9 \\
PointLLM-7B\textsuperscript{\textdagger}     & Full($\sim$7B)  & 3D point cloud & 41.2 & 53.4 \\
PointLLM-13B\textsuperscript{\textdagger}    & Full($\sim$13B) & 3D point cloud & 46.6 & 53.0 \\
ShapeLLM-7B\textsuperscript{\textdagger}     & Full($\sim$7B)  & 3D point cloud & 47.4 & 53.1 \\
ShapeLLM-13B\textsuperscript{\textdagger}    & Full($\sim$13B) & 3D point cloud & 53.1 & 53.0 \\
\midrule
Adapter & PEFT(9.5M) & 3D point cloud & 49.7 & 53.2 \\
MPM     & PEFT(78M)  & 3D point cloud & \textbf{57.6} & \textbf{70.5} \\
\bottomrule
\end{tabularx}
\vspace{-0.4cm}
\end{table}

\subsection{Experimental Setup}

\noindent\textbf{Backbone and Frozen Feature Extractors.}
All experiments are built on Qwen3-Omni-30B-A3B, which serves as the frozen MLLM backbone.
We keep the entire backbone \emph{fully frozen} throughout training: we never update the LLM parameters nor its native multimodal branches.
In particular, we use the image branch of Qwen3-Omni-30B-A3B as a fixed visual feature extractor for image inputs.
For newly introduced modalities, we adopt off-the-shelf encoders in baselines and freeze them as well: for 3D point clouds, we use PointBERTv1.2~\cite{pointbert} released by PointLLM~\cite{pointllm}; for tactile sensing, we use the tactile-pretrained ViT-Base released in TVL~\cite{tvl}.
Consequently, all trainable parameters in our system reside on the modality side (our alignment autoencoder and the Stage~II mappers), which preserves the backbone’s native multimodal capabilities and mitigates catastrophic forgetting.

\vspace{0.2cm}
\noindent\textbf{Modality-Side Autoencoder Architecture.}
We instantiate a lightweight autoencoder to learn a unified shared latent space from paired supervision and to support cross-modal recomposition.
Across text/image/tactile/point modalities, we use a unified encoder template: each modality encoder is a 4-layer attention stack that maps frozen modality tokens into a structured latent representation.
On the decoding side, we employ 12-layer attention decoders.

\vspace{0.2cm}
\noindent\textbf{Paired Training Data.}
We use three types of paired training data: text--3D, text--tactile, and image--tactile.
For text--3D, we combine two sources: (i) 660K Objaverse point clouds with captions provided by the PointLLM annotations, and (ii) an additional 100K synthetic point--caption pairs generated using Microsoft TRELLIS~\cite{trellis}.
For tactile, we use the TVL dataset and restrict supervision to touch--text and touch--image pairs only, resulting in 43K paired samples in total.

\subsection{Evaluation}
\label{sec:evaluation}

\noindent\textbf{3D Point Cloud Benchmarks.}
We evaluate 3D point understanding on two benchmarks: ModelNet40~\cite{modelnet}, a standard 40-way object classification dataset, and 3D MM-VET released by ShapeLLM~\cite{shapellm}, a 3D question answering benchmark with free-form responses.
On ModelNet40, we prompt the model to choose one label from the 40 canonical categories and report top-1 accuracy.
On 3D-MMVET, we follow the official protocol and use the benchmark’s GPT-based evaluation script to judge answer correctness.
We compare against the frozen Qwen3-Omni-30B-A3B backbone~\cite{xu2025qwen3omni}, several 3D-aware multimodal language models trained/fine-tuned with 3D--language supervision, and an adapter baseline implemented by us that inserts a two-layer MLP projector on top of the same Qwen3-Omni backbone.
MPM outperforms all baselines on both benchmarks.

\vspace{0.2cm}
\noindent\textbf{Tactile Benchmarks.}
For tactile sensing, we evaluate on TVL, where the input is a tactile image (optionally with an object image), and the model describes tactile attributes such as \texttt{soft} or \texttt{smooth}.
Following prior work, we use GPT-4 for score prediction, ground-truth agreement on a 1--10 scale.
As shown in \cref{tab:tactile_bench}, MPM achieves the best performance among tactile-capable baselines, and it does so \emph{without} requiring the auxiliary object image at inference time. The experiment details and showcases can be found in Appendix B.

\begin{table}[t]
\centering
\caption{\textbf{Performance comparison on TVL benchmarks.} Most results come from TVL~\cite{tvl}. All results are evaluated by GPT-4.}
\vspace{-0.2cm}
\label{tab:tactile_bench}
\setlength{\tabcolsep}{7pt}
\renewcommand{\arraystretch}{1.15}
\begin{tabular}{lccc}
\hline
Method & SSVTP & HCT & TVL \\
\hline
LLaVA-1.5 13B          & 3.55 & 3.63 & 3.62 \\
ViP-LLaVA 13B          & 4.10 & 3.76 & 3.80 \\
GPT-4V                 & 5.02 & 4.42 & 4.49 \\
TVL-LLaMA (ViT-Base)   & 6.16 & 4.89 & 5.03 \\
Qwen3-Omni             & 5.30 & 4.28 & 4.40 \\
\hline
MPM                   &  \textbf{6.24}  & \textbf{5.09}   & \textbf{5.22}   \\
\hline
\end{tabular}
\vspace{-0.4cm}
\end{table}

\subsection{Ablations}
\label{sec:ablation}

\noindent\textbf{Can we skip Stage~I and directly ``plug'' a new modality into a frozen MLLM via recomposition? No.}
We remove Stage~I entirely and only optimize the Stage~II interface to the frozen backbone.
This \emph{recomposition-only} variant performs substantially worse than the full model (Table~\ref{tab:ablation}), suggesting that under pairwise data it is difficult to simultaneously discover modality semantics and align them to the backbone embedding space in a single end-to-end step.
Stage~I mitigates this by learning an information-preserving shared latent code on which recomposition can reliably operate; without such a common latent space, the cross-modal ``translation'' becomes poorly constrained.

\vspace{0.2cm}
\noindent\textbf{If Stage~I already aligns representations, do we still need Stage~II? Yes.}
We keep Stage~I training intact but remove Stage~II, \ie, we do not connect the learned shared code to the frozen MLLM decoder.
This \emph{Stage~I only} variant underperforms the full pipeline (Table~\ref{tab:ablation}), indicating that aligned representations alone do not directly yield strong language-grounded performance.
Stage~I decoders are lightweight reconstruction heads and are not meant to model fluent language generation or instruction-following behavior; Stage~II is essential as a compatibility bridge that maps $\hat{\mathbf z}_c$ into the backbone’s input embedding space and leverages the frozen decoder for robust downstream performance.

\vspace{0.2cm}
\noindent\textbf{Which ingredient in Stage~I matters more: reconstruction or contrastive alignment? Both matter.}
To isolate the contributions of the two Stage~I objectives, we train two variants that remove either the self-modal reconstruction losses $\mathcal{L}^{(m)}_{\mathrm{rec}}$ or the pairwise contrastive losses $\mathcal{L}^{(ij)}_{\mathrm{align}}$, while keeping all other settings unchanged.
Removing reconstruction hurts performance more than removing contrastive alignment (Table~\ref{tab:ablation}), consistent with the intended design: reconstruction constrains the latent space to remain informative and prevents collapse, while contrastive learning provides the cross-modal signal that pulls paired modalities into a shared semantic coordinate system.
 More ablation studies and discussions can be found in Appendix C.

\begin{table}[t]
\centering
\caption{\textbf{Ablation study on our evaluation suite (higher is better).}
\textit{Drop} is the absolute decrease in ModelNet40 score relative to the full model.}
\vspace{-0.2cm}
\label{tab:ablation}
\begin{tabular}{lcc}
\toprule
Variant & ModelNet40 score (\%) $\uparrow$ & Drop (\%) $\downarrow$ \\
\midrule
Full (Stage~I + Stage~II) & 70.5 & 0.0 \\
w/o Stage~I (recomposition-only) & 47.3 & 23.2 \\
w/o Stage~II (Stage~I only) & 63.6 & 6.9 \\
w/o reconstruction $\mathcal{L}_{\mathrm{rec}}$ in Stage~I & 4.1 & 66.4 \\
w/o contrastive $\mathcal{L}_{\mathrm{align}}$ in Stage~I & 54.7 & 15.8 \\
\bottomrule
\end{tabular}
\vspace{-0.4cm}
\end{table}
\section{Conclusion}

In this work, we study MLLM training under {pairwise-aligned} supervision, without requiring fully jointly aligned multimodal data. We provide a theoretical analysis showing that, under suitable connectivity, the shared latent subspace across modalities can be recovered from overlapping modality pairs. Building on this insight, we propose a two-stage framework consisting of latent representation alignment and cross-modal recomposition, enabling scalable modality extension while keeping pre-trained MLLM frozen. Experiments with newly introduced 3D and tactile modalities demonstrate strong cross-modal transfer using only pairwise supervision. \textbf{Limitations.}
While our framework demonstrates promising results under pairwise supervision, scaling it to a broader set of modalities and validating its effectiveness on more diverse and complex downstream tasks remain important directions for future work.


\bibliographystyle{splncs04}
\bibliography{main}


\clearpage
\appendix

\crefname{section}{Appendix}{Appendices}
\Crefname{section}{Appendix}{Appendices}

\beginsupplement

\section{Proof and Discussions}
\label{proof}

We summarize the notations used throughout the paper in Table~\ref{tab:notation_proof_onepage}.

\begin{table*}[!t]
\centering
\footnotesize
\renewcommand{\arraystretch}{0.92}
\setlength{\tabcolsep}{6pt}
\caption{Key notations in this paper.}
\label{tab:notation_proof_onepage}
\begin{tabularx}{\textwidth}{@{}lX@{}}
\toprule
Notation & Meaning \\
\midrule
$M,\ [M]$ & Number of modalities; index set $\{1,\dots,M\}$. \\
$\vec{x}^{(m)}\in\mathbb{R}^{d_x^{(m)}}$ & Observation of modality $m$. \\
$\vec{z}^{(m)}=[\vec{z}^{(m)}_c,\vec{z}^{(m)}_s]$ & Latents of modality $m$ (shared/specific blocks). \\
$d_c^{(m)}, d_s^{(m)}$ & Dimensions of $\vec{z}^{(m)}_c$ and $\vec{z}^{(m)}_s$. \\
$\tilde{\vec{z}}^{(j)}_{\setminus i}$ & Latents for generating $\vec{x}^{(j)}$ excluding $\vec{z}^{(i)}_c$ (pairwise). \\
$\mathcal{G}=([M],\mathcal{E})$ & Modality graph; edge $\{i,j\}\in\mathcal{E}$ means paired data observed. \\
$\mathcal{N}(i)$ & Neighbor set $\{j\neq i:\{i,j\}\in\mathcal{E}\}$. \\
$g_i$ & Single-modality generator: $\vec{x}^{(i)}=g_i(\vec{z}^{(i)}_c,\vec{z}^{(i)}_s)$. \\
$g_{j\leftarrow i}$ & Pairwise generator: $\vec{x}^{(j)}=g_{j\leftarrow i}(\vec{z}^{(i)}_c,\tilde{\vec{z}}^{(j)}_{\setminus i})$. \\
$\hat g_i,\ \hat g_{j\leftarrow i}$ & Estimated generators. \\
$J_f$ & Jacobian of function $f$. \\
$I(\vec{z}^{(i)}_c)$ & Index set of $\vec{z}^{(i)}_c$ coordinates (for Jacobian slicing). \\
$I_{d_c^{(i)}}$ & Identity matrix on the shared block (use in Assumption~\ref{subspace_ass}\ref{subspace_ass:A2}). \\
$A_{j\leftarrow i}$ & Partial Jacobian $\bigl[J_{g_{j\leftarrow i}}\bigr]_{:,I(\vec{z}^{(i)}_c)}\in\mathbb{R}^{d_x^{(j)}\times d_c^{(i)}}$. \\
$L_{ij}$ & Coefficients such that $\sum_{j\in\mathcal{N}(i)} L_{ij}A_{j\leftarrow i}=I_{d_c^{(i)}}$. \\
$h_i=g_i^{-1}\circ \hat g_i$ & Coordinate map from estimated to true latents (modality $i$). \\
$h_i^{(c)},h_i^{(s)}$ & Shared/specific components of $h_i$. \\
$\varphi_i$ & Reduced map $\vec{z}^{(i)}_c=\varphi_i(\hat{\vec{z}}^{(i)}_c)$ (from Theorem~\ref{thm:multi_pair_ident}). \\
$\pi_m$ & Injective map $[\vec{z}^{(m)}_c]_k=c_{\pi_m(k)}$ linking to global shared factors. \\
$\vec{z}_e^{\,c}$ & Extended shared vector $[(\vec{z}^{(1)}_c)^\top,\dots,(\vec{z}^{(M)}_c)^\top]^\top$. \\
$\mathcal I_m$ & Indices of modality-$m$ block within $\vec{z}_e^{\,c}$. \\
$\mathcal I_{m\mid n}$ & Non-overlap indices of $m$ w.r.t.\ $n$ (for de-duplicated sparsity). \\
$g_c,\ G$ & SCM on $\vec{z}_e^{\,c}$ and its Jacobian $G=\partial g_c/\partial(\vec{z}_e^{\,c})^\top$. \\
$G^{m\leftarrow n}$ & Block $G_{\mathcal I_m,\mathcal I_n}$. \\
$\|M\|_{0,\mathcal E}$ & De-duplicated cross-modality sparsity counted only on edges $\mathcal E$. \\
$\mathcal P(d),\ \mathcal P_c$ & Generalized permutations; block-diagonal version across modalities. \\
$h_m,\ h$ & Block diffeomorphisms $h_m$ and their concatenation $h$ on $\vec{z}_e^{\,c}$. \\
$T=J_h$ & Jacobian of $h$; $T=\mathrm{diag}(T_1,\dots,T_M)$. \\
$\hat g_c,\ \hat G$ & Reparameterized SCM $\hat g_c=h^{-1}\circ g_c\circ h$ and its Jacobian $\hat G$. \\
$P_m,\ \phi_{m,k}$ & Component-wise ambiguity: permutation $P_m$ and 1D diffeomorphisms $\phi_{m,k}$. \\
\bottomrule
\end{tabularx}
\end{table*}


\subsection{Proof of Block Identifiability (Subspace Identifiability)}
\label{app:block_ident}

This appendix provides a detailed proof for Theorem ~\ref{thm:multi_pair_ident} and clarifies the meaning of
Assumption~\ref{subspace_ass}\ref{subspace_ass:A2} in the main text.
Throughout, we fix a target modality $i\in[M]$ and its neighbor set $\mathcal{N}(i)$ . We use the same notations as in the main text:
$\vec{x}^{(m)}$, $(\vec{z}^{(m)}_c,\vec{z}^{(m)}_s)$, $\tilde{\vec{z}}^{(j)}_{\setminus i}$, and
the partial Jacobian
\[
A_{j\leftarrow i}(\vec{z}^{(i)}_c,\tilde{\vec{z}}^{(j)}_{\setminus i})
:=
\bigl[J_{g_{j\leftarrow i}(\vec{z}^{(i)}_c,\tilde{\vec{z}}^{(j)}_{\setminus i})}\bigr]_{:,\,I(\vec{z}^{(i)}_c)}
\in\mathbb{R}^{d_x^{(j)}\times d_c^{(i)}}.
\]

First, we introduce the meanings of \textbf{Assumption~\ref{subspace_ass}\ref{subspace_ass:A2}}.
Assumption~\ref{subspace_ass}\ref{subspace_ass:A2} is a \emph{collective} full-rank condition:
each individual neighbor $j$ may only reveal a subspace of directions in $\mathbb{R}^{d_c^{(i)}}$,
but aggregating all neighbors must span the full shared block. The following lemma provides
equivalent characterizations.

\begin{lemma}[Equivalent collective full-rank characterizations]
\label{lem:equiv_collective}
Fix a target modality $i$ and its neighbor set $\mathcal{N}(i)$.
For each $j\in\mathcal{N}(i)$, let
$A_{j\leftarrow i}\in\mathbb{R}^{d_x^{(j)}\times d_c^{(i)}}$
be any matrices (e.g., the Jacobians $A_{j\leftarrow i}(\vec{z}^{(i)}_c,\tilde{\vec{z}}^{(j)}_{\setminus i})$
evaluated at a fixed latent realization).
Define the linear operator
\[
\mathcal{A}_i:\mathbb{R}^{d_c^{(i)}}\to \prod_{j\in\mathcal{N}(i)}\mathbb{R}^{d_x^{(j)}},
\qquad
\mathcal{A}_i(\mathbf{v})=\bigl(A_{j\leftarrow i}\mathbf{v}\bigr)_{j\in\mathcal{N}(i)}.
\]
Then the following are equivalent:
\begin{enumerate}
\item[(i)] (\textbf{Injectivity}) $\mathcal{A}_i$ is injective:
$\mathcal{A}_i(\mathbf{v})=0 \Rightarrow \mathbf{v}=0$.
Equivalently, \[\bigcap_{j\in\mathcal{N}(i)}\mathrm{Null}(A_{j\leftarrow i})=\{0\}\].
\item[(ii)] (\textbf{Gram positive definite})
The Gram matrix
\[
G_i:=\sum_{j\in\mathcal{N}(i)} A_{j\leftarrow i}^\top A_{j\leftarrow i}
\in\mathbb{R}^{d_c^{(i)}\times d_c^{(i)}}
\]
is positive definite (hence full rank).
\item[(iii)] (\textbf{Coefficient combination / left inverse})
There exist matrices
$L_{ij}\in\mathbb{R}^{d_c^{(i)}\times d_x^{(j)}}$ such that
\[
\sum_{j\in\mathcal{N}(i)} L_{ij}A_{j\leftarrow i}=I_{d_c^{(i)}}.
\]
Moreover, if $G_i\succ 0$, one valid choice is $L_{ij}=G_i^{-1}A_{j\leftarrow i}^\top$.
\end{enumerate}
\end{lemma}

\begin{proof}
Let $d:=d_c^{(i)}$ and choose an arbitrary ordering $\mathcal{N}(i)=\{j_1,\dots,j_K\}$.
Define the block-stacked matrix
\[
A_i \;\triangleq\;
\begin{bmatrix}
A_{j_1\leftarrow i}\\
\vdots\\
A_{j_K\leftarrow i}
\end{bmatrix}
\in \mathbb R^{D_i\times d},
\qquad
D_i:=\sum_{\ell=1}^K d_x^{(j_\ell)}.
\]
Also define the canonical concatenation isomorphism
\[
\Phi:\prod_{\ell=1}^K \mathbb R^{d_x^{(j_\ell)}}\to \mathbb R^{D_i},
\qquad
\Phi\bigl( (y_{j_\ell})_{\ell=1}^K \bigr)=
\begin{bmatrix}
y_{j_1}\\ \vdots \\ y_{j_K}
\end{bmatrix}.
\]
By construction, for every $\mathbf{v}\in\mathbb R^d$,
\[
\Phi\bigl(\mathcal A_i(\mathbf{v})\bigr)
=
\begin{bmatrix}
A_{j_1\leftarrow i}\mathbf{v}\\ \vdots\\ A_{j_K\leftarrow i}\mathbf{v}
\end{bmatrix}
=
A_i \mathbf{v}.
\]
Since $\Phi$ is a linear bijection, $\mathcal A_i(\mathbf{v})=0$ iff $A_i\mathbf{v}=0$.
Thus (i) holds iff $\mathrm{Null}(A_i)=\{0\}$, i.e., $A_i$ has full column rank.

\medskip\noindent
\textbf{(i)$\Leftrightarrow$(ii).}
Using the block-stacking definition,
\[
A_i^\top A_i
=
\sum_{\ell=1}^K A_{j_\ell\leftarrow i}^\top A_{j_\ell\leftarrow i}
=
\sum_{j\in\mathcal{N}(i)} A_{j\leftarrow i}^\top A_{j\leftarrow i}
=
G_i.
\]
Therefore, for any $\mathbf{v}\neq 0$,
\[
\mathbf{v}^\top G_i \mathbf{v}
=
\mathbf{v}^\top A_i^\top A_i \mathbf{v}
=
\|A_i\mathbf{v}\|_2^2.
\]
Hence $G_i\succ 0$ iff $A_i\mathbf{v}\neq 0$ for all $\mathbf{v}\neq 0$, which is equivalent to
$\mathrm{Null}(A_i)=\{0\}$, i.e., (i).

\medskip\noindent
\textbf{(ii)$\Rightarrow$(iii).}
Assume $G_i\succ 0$ and define $L_{ij}:=G_i^{-1}A_{j\leftarrow i}^\top$.
Then
\[
\sum_{j\in\mathcal{N}(i)} L_{ij}A_{j\leftarrow i}
=
\sum_{j\in\mathcal{N}(i)} G_i^{-1}A_{j\leftarrow i}^\top A_{j\leftarrow i}
=
G_i^{-1}G_i
=
I_{d}.
\]

\medskip\noindent
\textbf{(iii)$\Rightarrow$(i).}
If $\sum_{j}L_{ij}A_{j\leftarrow i}=I_d$ and $\mathcal A_i(\mathbf{v})=0$, then $A_{j\leftarrow i}\mathbf{v}=0$
for all $j$, so
\[
\mathbf{v}
=
I_d\mathbf{v}
=
\sum_{j\in\mathcal{N}(i)} L_{ij}A_{j\leftarrow i}\mathbf{v}
=
0.
\]
Thus $\mathcal A_i$ is injective, i.e., (i) holds.
\end{proof}

At any latent realization $(\vec{z}^{(i)}_c,\tilde{\vec{z}}^{(j)}_{\setminus i})_{j\in\mathcal{N}(i)}$,
Assumption~\ref{subspace_ass}\ref{subspace_ass:A2} asserts the existence of matrices
$L_{ij}(\vec{z}^{(i)}_c,\tilde{\vec{z}}^{(j)}_{\setminus i})$ such that
\(
\sum_{j\in\mathcal{N}(i)} L_{ij}A_{j\leftarrow i}=I_{d_c^{(i)}}.
\)
By Lemma~\ref{lem:equiv_collective}, this is equivalent to the stacked operator being injective,
or equivalently, the aggregated Gram matrix being positive definite, i.e.,
the neighbors collectively span all directions of the shared block.


\begin{proof}[Proof of \cref{thm:multi_pair_ident}]

Fix a target modality $i$ and its neighbor set $\mathcal{N}(i)$.

By Assumption~\ref{subspace_ass}\ref{subspace_ass:A1}, both $g_i$ and $\hat g_i$ are smooth and invertible.
Define the (smooth) change-of-variables map
\[
h_i
\;:=\;
g_i^{-1}\circ \hat g_i:
(\hat{\vec{z}}^{(i)}_c,\hat{\vec{z}}^{(i)}_s)\mapsto (\vec{z}^{(i)}_c,\vec{z}^{(i)}_s).
\]
Write its block decomposition as
\[
h_i(\hat{\vec{z}}^{(i)}_c,\hat{\vec{z}}^{(i)}_s)
=
\Bigl(
h_i^{(c)}(\hat{\vec{z}}^{(i)}_c,\hat{\vec{z}}^{(i)}_s),\;
h_i^{(s)}(\hat{\vec{z}}^{(i)}_c,\hat{\vec{z}}^{(i)}_s)
\Bigr),
\]
so that
\[
\vec{z}^{(i)}_c = h_i^{(c)}(\hat{\vec{z}}^{(i)}_c,\hat{\vec{z}}^{(i)}_s),
\qquad
\vec{z}^{(i)}_s = h_i^{(s)}(\hat{\vec{z}}^{(i)}_c,\hat{\vec{z}}^{(i)}_s).
\]

Fix any neighbor $j\in\mathcal{N}(i)$.
Under the estimated model, generate
\[
\vec{x}^{(i)}=\hat g_i(\hat{\vec{z}}^{(i)}_c,\hat{\vec{z}}^{(i)}_s),
\qquad
\vec{x}^{(j)}=\hat g_{j\leftarrow i}\bigl(\hat{\vec{z}}^{(i)}_c,\hat{\tilde{\vec{z}}}^{(j)}_{\setminus i}\bigr).
\]
By the theorem assumption, the true and estimated models induce identical pairwise distributions
for $(\vec{x}^{(i)},\vec{x}^{(j)})$.
Hence we may consider a coupling under which the same pair
$(\vec{x}^{(i)},\vec{x}^{(j)})$ is also generated by the true model, i.e., there exist
$(\vec{z}^{(i)}_c,\vec{z}^{(i)}_s,\tilde{\vec{z}}^{(j)}_{\setminus i})$ such that
\[
\vec{x}^{(i)}=g_i(\vec{z}^{(i)}_c,\vec{z}^{(i)}_s),
\qquad
\vec{x}^{(j)}=g_{j\leftarrow i}\bigl(\vec{z}^{(i)}_c,\tilde{\vec{z}}^{(j)}_{\setminus i}\bigr),
\]
and moreover, by definition of $h_i$,
\[
\vec{z}^{(i)}_c = h_i^{(c)}(\hat{\vec{z}}^{(i)}_c,\hat{\vec{z}}^{(i)}_s).
\]
Consequently, almost surely under this coupling,
\begin{equation}
\label{eq:app_key_identity_xj}
\hat g_{j\leftarrow i}\bigl(\hat{\vec{z}}^{(i)}_c,\hat{\tilde{\vec{z}}}^{(j)}_{\setminus i}\bigr)
=
g_{j\leftarrow i}\!\left(
h_i^{(c)}(\hat{\vec{z}}^{(i)}_c,\hat{\vec{z}}^{(i)}_s),\;
\tilde{\vec{z}}^{(j)}_{\setminus i}
\right).
\end{equation}

Differentiate both sides of \eqref{eq:app_key_identity_xj} with respect to $\hat{\vec{z}}^{(i)}_s$.
The left-hand side is identically independent of $\hat{\vec{z}}^{(i)}_s$, hence its derivative is zero:
\[
\frac{\partial}{\partial \hat{\vec{z}}^{(i)}_s}\;
\hat g_{j\leftarrow i}\bigl(\hat{\vec{z}}^{(i)}_c,\hat{\tilde{\vec{z}}}^{(j)}_{\setminus i}\bigr)
=0.
\]

For the right-hand side, we need to control the dependence of $\tilde{\vec{z}}^{(j)}_{\setminus i}$
on $\hat{\vec{z}}^{(i)}_s$.
Recall that $\tilde{\vec{z}}^{(j)}_{\setminus i}$ is a subvector of the latent variables of modality $j$
.
Applying Assumption~\ref{subspace_ass}\ref{subspace_ass:A1} to modality $j$ yields that $g_j$ is invertible,
so $\vec{z}^{(j)}=g_j^{-1}(\vec{x}^{(j)})$ is a smooth function of $\vec{x}^{(j)}$ alone. Therefore
$\tilde{\vec{z}}^{(j)}_{\setminus i}$, being a subvector (or smooth projection) of $\vec{z}^{(j)}$,
is also a smooth function of $\vec{x}^{(j)}$ alone. Since $\vec{x}^{(j)}$ on the left-hand side of
\eqref{eq:app_key_identity_xj} does not depend on $\hat{\vec{z}}^{(i)}_s$, we have
\[
\frac{\partial \tilde{\vec{z}}^{(j)}_{\setminus i}}{\partial \hat{\vec{z}}^{(i)}_s}=0.
\]

Thus, differentiating \eqref{eq:app_key_identity_xj} and applying the chain rule gives
\[
0
=
A_{j\leftarrow i}\bigl(\vec{z}^{(i)}_c,\tilde{\vec{z}}^{(j)}_{\setminus i}\bigr)\;
\frac{\partial\, h_i^{(c)}(\hat{\vec{z}}^{(i)}_c,\hat{\vec{z}}^{(i)}_s)}{\partial \hat{\vec{z}}^{(i)}_s}.
\]
Define
\[
B_i(\hat{\vec{z}}^{(i)}_c,\hat{\vec{z}}^{(i)}_s)
:=
\frac{\partial\, h_i^{(c)}(\hat{\vec{z}}^{(i)}_c,\hat{\vec{z}}^{(i)}_s)}{\partial \hat{\vec{z}}^{(i)}_s}
\in\mathbb{R}^{d_c^{(i)}\times d_s^{(i)}}.
\]
Then for every $j\in\mathcal{N}(i)$,
\begin{equation}
\label{eq:app_A_B_zero}
A_{j\leftarrow i}\bigl(\vec{z}^{(i)}_c,\tilde{\vec{z}}^{(j)}_{\setminus i}\bigr)\;B_i(\hat{\vec{z}}^{(i)}_c,\hat{\vec{z}}^{(i)}_s)=0.
\end{equation}

Now left-multiply \eqref{eq:app_A_B_zero} by
$L_{ij}(\vec{z}^{(i)}_c,\tilde{\vec{z}}^{(j)}_{\setminus i})$ and sum over $j\in\mathcal{N}(i)$:
\[
\sum_{j\in\mathcal{N}(i)}
L_{ij}\bigl(\vec{z}^{(i)}_c,\tilde{\vec{z}}^{(j)}_{\setminus i}\bigr)\;
A_{j\leftarrow i}\bigl(\vec{z}^{(i)}_c,\tilde{\vec{z}}^{(j)}_{\setminus i}\bigr)\;
B_i(\hat{\vec{z}}^{(i)}_c,\hat{\vec{z}}^{(i)}_s)
=0.
\]
By Assumption~\ref{subspace_ass}\ref{subspace_ass:A2},
\[
\sum_{j\in\mathcal{N}(i)}
L_{ij}\bigl(\vec{z}^{(i)}_c,\tilde{\vec{z}}^{(j)}_{\setminus i}\bigr)\;
A_{j\leftarrow i}\bigl(\vec{z}^{(i)}_c,\tilde{\vec{z}}^{(j)}_{\setminus i}\bigr)
=
I_{d_c^{(i)}},
\]
hence
\[
B_i(\hat{\vec{z}}^{(i)}_c,\hat{\vec{z}}^{(i)}_s)=0
\qquad\text{(a.e.\ on the latent support)}.
\]
Therefore, $h_i^{(c)}(\hat{\vec{z}}^{(i)}_c,\hat{\vec{z}}^{(i)}_s)$ does \emph{not} depend on
$\hat{\vec{z}}^{(i)}_s$, i.e., there exists a smooth function $\varphi_i$ such that
\[
\vec{z}^{(i)}_c
=
h_i^{(c)}(\hat{\vec{z}}^{(i)}_c,\hat{\vec{z}}^{(i)}_s)
=
\varphi_i(\hat{\vec{z}}^{(i)}_c).
\]

Since $h_i=g_i^{-1}\circ \hat g_i$ is a diffeomorphism, its Jacobian is invertible everywhere.
Because $\partial h_i^{(c)}/\partial \hat{\vec{z}}^{(i)}_s=0$, the Jacobian of $h_i$ is block lower-triangular:
\[
J_{h_i}(\hat{\vec{z}}^{(i)}_c,\hat{\vec{z}}^{(i)}_s)
=
\begin{bmatrix}
J_{\varphi_i}(\hat{\vec{z}}^{(i)}_c) & 0\\
* & *
\end{bmatrix},
\]
which implies $\det(J_{\varphi_i}(\hat{\vec{z}}^{(i)}_c))\neq 0$ for all points on the support.
By the inverse function theorem, $\varphi_i$ is locally invertible (a local diffeomorphism)
on the shared latent support. Hence there exists a smooth inverse map $h^{(i)}:=\varphi_i^{-1}$
(on the support) such that
\[
\hat{\vec{z}}^{(i)}_c
=
h^{(i)}\bigl(\vec{z}^{(i)}_c\bigr),
\]
which matches the subspace/block identifiability notion stated in the main text.
This completes the proof.
\end{proof}

\subsection{Component-wise Identifiability}
\label{subsec:component_ident}

Theorem~\ref{thm:multi_pair_ident} establishes {block/subspace} identifiability of the shared latent block:
for each modality $m$, the learned shared representation can match the ground-truth one up to an invertible change of
coordinates, i.e., $\hat{\vec{z}}^{(m)}_c = h^{(m)}(\vec{z}^{(m)}_c)$ (equivalently $\vec{z}^{(m)}_c = \tilde h^{(m)}(\hat{\vec{z}}^{(m)}_c)$)
for some diffeomorphism $h^{(m)}$.
This is already sufficient for {cross-modal alignment at the subspace level} (e.g., transfer, retrieval, or matching),
but it still allows arbitrary {within-block mixing}: a single learned coordinate may entangle multiple true factors.
Such an ambiguity limits {fine-grained controllability} and {interpretability} of generative models:
for controllable cross-modal generation (e.g., changing one semantic factor and generating consistent changes in image/text/tactile),
we would like a {one-to-one correspondence} between latent {components} across modalities, rather than only a shared subspace.

This motivates a stronger notion---{component-wise identifiability}---which aims to reduce the remaining ambiguity from
an arbitrary diffeomorphism to only {(i) within-block permutation} and {(ii) element-wise invertible scalings}.
As discussed in the disentanglement literature, such a stronger conclusion is impossible without additional inductive bias;
a widely used and practically meaningful bias is some form of {sparsity/structural simplicity} in the underlying mechanisms
\cite{zhang2024causal}.

We say $\hat{\vec{z}}^{(m)}_c$ is {component-wise identifiable} w.r.t.\ $\vec{z}^{(m)}_c$ if there exist
a permutation matrix $P_m$ and element-wise diffeomorphisms $\{\phi_{m,k}\}_{k=1}^{d_c^{(m)}}$ such that, on the support,
\[
\hat{\vec{z}}^{(m)}_c
=
P_m\,\phi_m(\vec{z}^{(m)}_c),
\qquad
\phi_m(\vec{u})
=
\bigl(\phi_{m,1}(u_1),\dots,\phi_{m,d_c^{(m)}}(u_{d_c^{(m)}})\bigr).
\]
Equivalently, $\vec{z}^{(m)}_c = \phi_m^{-1}\!\bigl(P_m^\top \hat{\vec{z}}^{(m)}_c\bigr)$.
Under this definition, each learned component corresponds to {exactly one} ground-truth shared factor (up to 1D reparameterization),
enabling {factor-level interventions} and {controllable} cross-modal generation.

We emphasize that the component-wise identifiability result in this appendix is \textbf{not} the main contribution of our work. Similar identifiability guarantees under closely related settings have already been established in prior studies~\cite{sun2024causal}. Our contribution differs from existing results in two key aspects.

First, we consider a more general structural case in which a single latent factor may simultaneously explain multiple observed modality pairs, \ie, it can be incident to more than one edge in the observation graph. This departs from the standard assumption that latent components align to disjoint modality pairs. To accommodate such shared factors, we introduce an extended latent representation and impose an appropriate sparsity structure that captures how each component participates across edges.
Second, our analysis starts from paired data only. Compared with settings that assume richer supervision (for example, multi-way aligned tuples), restricting to pairwise observations makes the problem fundamentally more under-determined. This motivates additional assumptions on the observation model to rule out ambiguous solutions and to ensure that the latent decomposition remains identifiable from pairwise measurements.

The overall proof strategy \textbf{follows} the line of reasoning in~\cite{sun2024causal}, while adapting the arguments to our extended latent construction and the pairwise-only supervision regime. 

Beyond establishing the formal statement, our goal in this section is to provide the community with clearer insight into what kinds of identifiability guarantees for multimodal latent variables are attainable when learning is driven purely by paired data.

Starting from block-identifiable solutions, we further introduce a {sparsity-based selection principle} on the
{cross-modality interactions} in the extended shared coordinates $\vec{z}^{\,c}_e$ (after carefully removing duplicated/overlapping
coordinates across modalities via the non-overlap index sets $\mathcal I_{m\mid n}$).
Intuitively, if the true cross-modality mechanism is sparse, then nontrivial within-block mixing tends to {spread} each cross-modality
dependency across many coordinates, producing a denser cross-block Jacobian; thus, selecting the sparsest mechanism (under mild coverage
conditions on observed pairs) rules out such mixing and leaves only permutations and element-wise transformations, formalized by
Theorem~\ref{thm:component_ident_dedup}.

This section strengthens the block/subspace identifiability result in
Theorem~\ref{thm:multi_pair_ident} to {component-wise} identifiability
under additional structural assumptions.
Throughout, modalities are indexed by $[M]=\{1,\dots,M\}$ and the observation graph is
$\mathcal{G}=([M],\mathcal{E})$ as in Section~2 of the main text.

Recall the notation in the main text:
the set of all connected (distinct) shared latent factors is
\[
\vec{c}=\{c_1,c_2,\dots,c_{d_c}\},
\qquad
\vec{c}
=
\bigcup_{m=1}^M
\bigl\{[\vec{z}^{(m)}_c]_k:\ k\in[d_c^{(m)}]\bigr\},
\]
and the set of all modality-specific latent variables is
$\vec{s}=\{s_1,s_2,\dots,s_{d_s}\}$.
In particular, each modality-wise shared block $\vec{z}^{(m)}_c$ is a (possibly permuted)
sub-collection of the global shared factors $\vec{c}$.
Equivalently, there exists an injective index map $\pi_m:[d_c^{(m)}]\to[d_c]$ such that
\begin{equation}
\label{eq:pi_map_shared}
[\vec{z}^{(m)}_c]_k = c_{\pi_m(k)},\qquad k\in[d_c^{(m)}],\ \ m\in[M].
\end{equation}
This section focuses on identifying components of $\vec{z}^{(m)}_c$; the modality-specific
variables $\vec{z}^{(m)}_s$ (whose union forms $\vec{s}$) play no role here.

In the learned representation, we directly concatenate the (possibly overlapping)
shared blocks $\{\vec{z}^{(m)}_c\}_{m=1}^M$ and form the extended vector
\[
\vec{z}_e^{\,c}
\triangleq
\bigl[(\vec{z}^{(1)}_c)^\top,\dots,(\vec{z}^{(M)}_c)^\top\bigr]^\top
\in \mathbb R^{d_e},
\qquad
d_e=\sum_{m=1}^M d_c^{(m)}.
\]
Let $o_1:=0$ and $o_m:=\sum_{\ell=1}^{m-1} d_c^{(\ell)}$ for $m\ge 2$, so that the
coordinates of $\vec{z}^{(m)}_c$ occupy the contiguous index set
\[
\mathcal I_m
\triangleq
\{o_m+1,\dots,o_m+d_c^{(m)}\}
\subset[d_e].
\]

Using the index maps $\{\pi_m\}_{m=1}^M$ in \eqref{eq:pi_map_shared}, for each ordered pair $(m,n)$,
define the indices of coordinates in $\vec{z}^{(m)}_c$ that correspond to global factors also
appearing in $\vec{z}^{(n)}_c$:
\[
\mathcal I_{m,n}
\triangleq
\Bigl\{\, o_m+k\in\mathcal I_m:\ \pi_m(k)\in \pi_n([d_c^{(n)}]) \Bigr\}.
\]
(When $\{m,n\}\notin\mathcal E$, the set $\mathcal I_{m,n}$ is still well-defined; it simply reflects
whether the same global factor $c_r\in\vec{c}$ is present in both modality-wise shared blocks.)
For $m\neq n$, define the non-overlap indices
\[
\mathcal I_{m\mid n}\triangleq \mathcal I_m\setminus \mathcal I_{m,n},
\qquad
\mathcal I_{n\mid m}\triangleq \mathcal I_n\setminus \mathcal I_{n,m}.
\]

We consider a latent causal mechanism (SCM) on the extended shared coordinates,
written in vector form as
\[
\vec{z}_e^{\,c} = g_c(\vec{z}_e^{\,c},\epsilon),
\]
where $g_c$ is {not} the observation generator $g_m$, but
a compact notation for the structural assignments among shared latents.
Define its Jacobian
\[
G \triangleq
\frac{\partial g_c(\vec{z}_e^{\,c},\epsilon)}{\partial (\vec{z}_e^{\,c})^\top}
\in\mathbb R^{d_e\times d_e},
\qquad
G^{m\leftarrow n}\triangleq G_{\mathcal I_m,\mathcal I_n}.
\]

For any matrix-valued random variable $M(\vec{z}_e^{\,c})$, define
\[
\|M\|_0
\triangleq
\sum_{m\neq n}\,
\bigl\|M_{\mathcal I_{m\mid n},\,\mathcal I_{n\mid m}}\bigr\|_0,
\qquad
\|A\|_0 \triangleq |\mathrm{supp}(A)|.
\]
Here
\(
\mathrm{supp}(A):=\{(u,v):\mathbb P(A_{uv}\neq 0)>0\}.
\)
For observed pairs only, define $\|M\|_{0,\mathcal E}$ as the same quantity but with the
outer sum restricted to $\{m,n\}\in\mathcal E$.

Let $\mathcal P(d)$ be the set of generalized permutation matrices in $\mathbb R^{d\times d}$ and define
\[
\mathcal P_c
\triangleq
\Bigl\{\mathrm{diag}(P_1,\dots,P_M):\ P_m\in\mathcal P(d_c^{(m)})\Bigr\}.
\]

\begin{assumptionbox}[Pair coverage of nonzero cross blocks]
\label{ass:pair_coverage_dedup}
If $\{m,n\}\notin\mathcal E$, then there is no cross-modality dependence between the
\emph{non-overlap} parts of modalities $m$ and $n$, so $(m,n)$ contributes zero to $\|G\|_0$.
Equivalently,
\[
\|G\|_0 \;=\; \|G\|_{0,\mathcal E}.
\]
\end{assumptionbox}

\begin{assumptionbox}[Global sparsity minimality]
\label{ass:global_sparse_minimal_dedup}
For almost every $\vec{z}_e^{\,c}$, for any block-diagonal invertible
$T=\mathrm{diag}(T_1,\dots,T_M)$,
if $T\notin\mathcal P_c$ then
\[
\|T^{-1}GT\|_0 \;>\; \|G\|_0.
\]
\end{assumptionbox}

\begin{theorembox}[Pairwise component-wise identifiability]
\label{thm:component_ident_dedup}
Assume block identifiability holds for every modality's shared block
(e.g., Theorem~\ref{thm:multi_pair_ident} applied with target $i=m$ for each $m\in[M]$),
Assumption~\ref{ass:pair_coverage_dedup} holds, and
Assumption~\ref{ass:global_sparse_minimal_dedup} holds.
Let $\hat{\vec{z}}^{(m)}_c$ be any block-identified representation consistent with the observed
pairwise marginals on $\mathcal E$.
Then there exist block-wise diffeomorphisms $h_m$ such that
\[
\vec{z}^{(m)}_c = h_m\!\bigl(\hat{\vec{z}}^{(m)}_c\bigr),\qquad m\in[M].
\]
Define
\[
h(\hat{\vec{z}}_e^{\,c})
\triangleq
\bigl[h_1(\hat{\vec{z}}^{(1)}_c)^\top,\dots,h_M(\hat{\vec{z}}^{(M)}_c)^\top\bigr]^\top,
\qquad
T(\hat{\vec{z}}_e^{\,c})\triangleq J_h(\hat{\vec{z}}_e^{\,c}).
\]
Consider the reparameterized mechanism
\[
\hat g_c(\hat{\vec{z}}_e^{\,c})
\triangleq
h^{-1}\!\bigl(g_c(h(\hat{\vec{z}}_e^{\,c}),\epsilon)\bigr),
\qquad
\hat G\triangleq \frac{\partial \hat g_c}{\partial (\hat{\vec{z}}_e^{\,c})^\top}.
\]
Assume the learned solution satisfies the sparsity selection inequality on observed pairs:
\[
\|\hat G\|_{0,\mathcal E}\ \le\ \|G\|_{0,\mathcal E}.
\]
Then $T(\hat{\vec{z}}_e^{\,c})\in\mathcal P_c$ almost surely; equivalently, each $h_m$ is
component-wise up to a within-block permutation (i.e., no nontrivial within-block mixing remains).
\end{theorembox}

Theorem~\ref{thm:component_ident_dedup} upgrades block identifiability to component-wise identifiability
by leveraging a structural preference encoded through the sparsity of cross-modality interactions on the
extended coordinates.
Assumption~\ref{ass:pair_coverage_dedup} enforces consistency with the observation graph:
after removing overlap coordinates that correspond to the same global factors in $\vec{c}$,
any remaining cross-modality dependence should only appear on modality pairs that are actually observed,
so the relevant nonzero patterns are entirely captured by edges in $\mathcal E$.
This is crucial because our selection criterion is evaluated only on observed pairs;
without such a coverage condition, a reparameterization could shift dependence into unobserved pairs
without being penalized.

Assumption~\ref{ass:global_sparse_minimal_dedup} imposes a uniqueness principle:
among all block-wise invertible coordinate changes, the true representation is a strict minimizer of the
cross-modality sparsity pattern on non-overlap parts, up to block-wise generalized permutations.
Intuitively, mixing coordinates within a modality tends to spread a single cross-modality influence direction
into multiple directions, thereby increasing the number of nonzero cross-modality entries when measured on the
non-overlapping parts $\mathcal I_{m\mid n},\mathcal I_{n\mid m}$.
The theorem then formalizes a selection argument:
starting from any block-identified solution consistent with pairwise marginals, if the learned mechanism is chosen
so that it is no denser than the true one on observed pairs, then any within-block mixing that would increase
cross-modality density is excluded, forcing the remaining ambiguity to be only block-wise generalized permutations.

\subsection{Proof of component-wise identifiability}
\label{app:component_ident_proof}

\begin{proof}[Proof of Theorem~\ref{thm:component_ident_dedup}]
By the assumed block identifiability (Theorem~\ref{thm:multi_pair_ident} applied to each modality),
for every $m\in[M]$ there exists a diffeomorphism
$h_m:\mathbb R^{d_c^{(m)}}\to \mathbb R^{d_c^{(m)}}$ such that
\begin{equation}
\label{eq:block_diffeo_app}
\vec{z}_c^{(m)} = h_m\!\bigl(\hat{\vec{z}}_c^{(m)}\bigr)\qquad\text{a.s.}
\end{equation}
Define the concatenated mapping
\[
h(\hat{\vec{z}}_e^{\,c})
\triangleq
\bigl[h_1(\hat{\vec{z}}_c^{(1)})^\top,\dots,h_M(\hat{\vec{z}}_c^{(M)})^\top\bigr]^\top,
\qquad
\vec{z}_e^{\,c} = h(\hat{\vec{z}}_e^{\,c})\ \text{a.s.}
\]
Since each $h_m$ is a diffeomorphism, $h$ is a diffeomorphism on $\mathbb R^{d_e}$ and its Jacobian
\[
\begin{aligned}
T(\hat{\vec{z}}_e^{\,c})
~\triangleq~
J_h(\hat{\vec{z}}_e^{\,c})
&=
\mathrm{diag}\!\bigl(J_{h_1}(\hat{\vec{z}}_c^{(1)}),\dots,J_{h_M}(\hat{\vec{z}}_c^{(M)})\bigr) \\
&=
\mathrm{diag}(T_1,\dots,T_M)
\end{aligned}
\]
is block-diagonal and invertible a.s., where $T_m\triangleq J_{h_m}(\hat{\vec{z}}_c^{(m)})$.

Recall the (extended-coordinate) latent mechanism
\[
\vec{z}_e^{\,c} = g_c(\vec{z}_e^{\,c},\epsilon),
\qquad
G \triangleq \frac{\partial g_c(\vec{z}_e^{\,c},\epsilon)}{\partial (\vec{z}_e^{\,c})^\top}.
\]
Define the reparameterized mechanism (as in Theorem~\ref{thm:component_ident_dedup})
\[
\hat g_c(\hat{\vec{z}}_e^{\,c},\epsilon)
\triangleq
h^{-1}\!\bigl(g_c(h(\hat{\vec{z}}_e^{\,c}),\epsilon)\bigr),
\qquad
\hat G
\triangleq
\frac{\partial \hat g_c(\hat{\vec{z}}_e^{\,c},\epsilon)}{\partial (\hat{\vec{z}}_e^{\,c})^\top}.
\]
We show that $\hat G$ is a conjugate (similarity) transform of $G$ by $T$.

Since $\vec{z}_e^{\,c}=h(\hat{\vec{z}}_e^{\,c})$ and $\vec{z}_e^{\,c}=g_c(\vec{z}_e^{\,c},\epsilon)$ hold
on observational samples, substituting gives
\[
g_c(h(\hat{\vec{z}}_e^{\,c}),\epsilon) = h(\hat{\vec{z}}_e^{\,c})
\qquad\text{a.s.}
\]
Applying $h^{-1}$ to both sides yields
\[
\hat g_c(\hat{\vec{z}}_e^{\,c},\epsilon)
=
h^{-1}(h(\hat{\vec{z}}_e^{\,c}))
=
\hat{\vec{z}}_e^{\,c}
\qquad\text{a.s.}
\]
Differentiate $\hat g_c(\hat{\vec{z}}_e^{\,c},\epsilon)
=
h^{-1}(g_c(h(\hat{\vec{z}}_e^{\,c}),\epsilon))$
with respect to $\hat{\vec{z}}_e^{\,c}$ and apply the chain rule:
\begin{align}
\hat G
&=
J_{h^{-1}}\!\bigl(g_c(h(\hat{\vec{z}}_e^{\,c}),\epsilon)\bigr)\;
J_{g_c}\!\bigl(h(\hat{\vec{z}}_e^{\,c}),\epsilon\bigr)\;
J_h(\hat{\vec{z}}_e^{\,c}).
\label{eq:chain_rule_full_app}
\end{align}
On observational samples we have $g_c(h(\hat{\vec{z}}_e^{\,c}),\epsilon)=h(\hat{\vec{z}}_e^{\,c})$, so the first Jacobian is
$J_{h^{-1}}(h(\hat{\vec{z}}_e^{\,c}))$.
By the inverse function theorem,
\[
J_{h^{-1}}(h(\hat{\vec{z}}_e^{\,c})) = \bigl(J_h(\hat{\vec{z}}_e^{\,c})\bigr)^{-1}
= T(\hat{\vec{z}}_e^{\,c})^{-1}.
\]
Also $J_{g_c}(h(\hat{\vec{z}}_e^{\,c}),\epsilon)$ is exactly $G$ evaluated at $\vec{z}_e^{\,c}=h(\hat{\vec{z}}_e^{\,c})$.
Plugging into \eqref{eq:chain_rule_full_app} yields the desired conjugacy:
\begin{equation}
\label{eq:similarity_app}
\hat G
=
T(\hat{\vec{z}}_e^{\,c})^{-1}\;G\;T(\hat{\vec{z}}_e^{\,c})
\qquad\text{a.s.}
\end{equation}
In particular, for any block $(m,n)$ with $m\neq n$,
\begin{equation}
\label{eq:block_similarity_app}
\hat G^{m\leftarrow n}
=
\hat G_{\mathcal I_m,\mathcal I_n}
=
T_m^{-1}\,G^{m\leftarrow n}\,T_n.
\end{equation}

Recall the de-duplicated sparsity functionals (defined in \cref{subsec:component_ident}):
\[
\|M\|_0
\triangleq
\sum_{m\neq n}\bigl\|M_{\mathcal I_{m\mid n},\,\mathcal I_{n\mid m}}\bigr\|_0,
\qquad
\|M\|_{0,\mathcal E}
\triangleq
\sum_{\substack{m\neq n:\\ \{m,n\}\in\mathcal E}}
\bigl\|M_{\mathcal I_{m\mid n},\,\mathcal I_{n\mid m}}\bigr\|_0,
\]
where for any matrix-valued random variable $A$,
$\|A\|_0 = |\mathrm{supp}(A)|$ and
$\mathrm{supp}(A)=\{(u,v):\mathbb P(A_{uv}\neq 0)>0\}$.

Assumption~\ref{ass:pair_coverage_dedup} states that if $\{m,n\}\notin\mathcal E$, then the
(non-overlap) cross-modality dependence between $m$ and $n$ vanishes, hence
\begin{equation}
\label{eq:G_full_equals_edge_app}
\|G\|_0 = \|G\|_{0,\mathcal E}.
\end{equation}

We also need the analogous statement for $\hat G$.
Under the bookkeeping convention used in \cref{subsec:component_ident} that
$\mathcal I_{m,n}=\emptyset$ whenever $\{m,n\}\notin\mathcal E$, we have
$\mathcal I_{m\mid n}=\mathcal I_m$ and $\mathcal I_{n\mid m}=\mathcal I_n$ for all non-edges.
Thus, for any $\{m,n\}\notin\mathcal E$,
Assumption~\ref{ass:pair_coverage_dedup} implies
$G_{\mathcal I_m,\mathcal I_n}\equiv 0$ and therefore, by \eqref{eq:block_similarity_app},
\[
\hat G_{\mathcal I_m,\mathcal I_n}
=
T_m^{-1}\,G_{\mathcal I_m,\mathcal I_n}\,T_n
\equiv 0.
\]
Hence $\hat G_{\mathcal I_{m\mid n},\,\mathcal I_{n\mid m}}\equiv 0$ for all non-edges, and summing gives
\begin{equation}
\label{eq:Ghat_full_equals_edge_app}
\|\hat G\|_0 = \|\hat G\|_{0,\mathcal E}.
\end{equation}

From \eqref{eq:similarity_app} we have $\hat G = T^{-1}GT$ a.s., hence
\begin{equation}
\label{eq:norm_transfer_app}
\|\hat G\|_0 = \|T^{-1}GT\|_0,
\qquad
\|\hat G\|_{0,\mathcal E} = \|T^{-1}GT\|_{0,\mathcal E}.
\end{equation}
Assumption~\ref{ass:global_sparse_minimal_dedup} states that for almost every $\vec{z}_e^{\,c}$,
for any block-diagonal invertible $S=\mathrm{diag}(S_1,\dots,S_M)$,
if $S\notin\mathcal P_c$ then
\begin{equation}
\label{eq:minimality_app}
\|S^{-1}GS\|_0 > \|G\|_0.
\end{equation}
Apply \eqref{eq:minimality_app} to $S=T(\hat{\vec{z}}_e^{\,c})$.
If $T(\hat{\vec{z}}_e^{\,c})\notin\mathcal P_c$ on an event of positive probability, then
\[
\|T^{-1}GT\|_0 > \|G\|_0.
\]
Using \eqref{eq:G_full_equals_edge_app} and \eqref{eq:Ghat_full_equals_edge_app}, and then \eqref{eq:norm_transfer_app}, we obtain
\[
\|\hat G\|_{0,\mathcal E}
=
\|\hat G\|_0
=
\|T^{-1}GT\|_0
>
\|G\|_0
=
\|G\|_{0,\mathcal E},
\]
which contradicts the sparsity selection inequality assumed in Theorem~\ref{thm:component_ident_dedup}:
$\|\hat G\|_{0,\mathcal E}\le \|G\|_{0,\mathcal E}$.
Therefore,
\begin{equation}
\label{eq:T_in_Pc_app}
T(\hat{\vec{z}}_e^{\,c})\in\mathcal P_c
\qquad\text{a.s.}
\end{equation}

Fix a modality $m$ and write $d=d_c^{(m)}$.
From \eqref{eq:T_in_Pc_app}, we have $T_m(\hat{\vec{z}}_c^{(m)})=J_{h_m}(\hat{\vec{z}}_c^{(m)})\in \mathcal P(d)$ a.s.
Because $h_m$ is smooth, $J_{h_m}(z)$ is continuous in $z$.
The set $\mathcal P(d)$ has finitely many \emph{support patterns} (one per permutation).
Thus, on each connected component of the support of $\hat{\vec{z}}_c^{(m)}$, the support pattern of
$J_{h_m}(z)$ must be constant a.s.
Consequently, there exists a permutation $\sigma_m$ of $[d]$ such that
\begin{equation}
\label{eq:fixed_perm_pattern_app}
\frac{\partial [h_m]_r(z)}{\partial [z]_k}=0
\quad\text{whenever }k\neq \sigma_m(r),
\qquad
r\in[d],
\ \ \text{for a.e. }z.
\end{equation}
(Additionally, since $J_{h_m}(z)$ is invertible, the remaining partial derivatives
$\frac{\partial [h_m]_r(z)}{\partial [z]_{\sigma_m(r)}}$ are nonzero a.e.)

Equation~\eqref{eq:fixed_perm_pattern_app} implies that each coordinate function $[h_m]_r$
depends on \emph{only one} input coordinate $[z]_{\sigma_m(r)}$ (a.e.):
for any $k\neq \sigma_m(r)$, the partial derivative in direction $k$ vanishes (a.e.),
so $[h_m]_r$ is (a.e.) constant along lines parallel to the $k$-axis.
Hence there exist univariate functions $\varphi_{m,r}:\mathbb R\to\mathbb R$ such that
\begin{equation}
\label{eq:univariate_form_app}
[h_m]_r(z) = \varphi_{m,r}\bigl([z]_{\sigma_m(r)}\bigr),\qquad r\in[d].
\end{equation}
Moreover, because $\frac{\partial [h_m]_r}{\partial [z]_{\sigma_m(r)}}\neq 0$ a.e.\ and $h_m$ is bijective,
each $\varphi_{m,r}$ is a 1D diffeomorphism on the support (and can be extended to a global 1D diffeomorphism
under standard regularity conditions).
Therefore $h_m$ is component-wise up to a permutation:
\[
h_m(z)
=
\Bigl(\varphi_{m,1}(z_{\sigma_m(1)}),\dots,\varphi_{m,d}(z_{\sigma_m(d)})\Bigr).
\]

We have shown that $T(\hat{\vec{z}}_e^{\,c})=J_h(\hat{\vec{z}}_e^{\,c})\in\mathcal P_c$ a.s.,
which forces each $h_m$ to be component-wise up to a within-block permutation.
Since $\vec{z}_c^{(m)}=h_m(\hat{\vec{z}}_c^{(m)})$, the ambiguity remaining after block identification
is only within-block permutation and a 1D invertible reparameterization of each coordinate,
i.e.\ component-wise identifiability (up to permutations) holds.
\end{proof}

\subsection{Discussions}
\label{subsec:discussions}

\subsubsection{Why the block-level identifiability assumptions are reasonable and widely used}

\paragraph{Assumption~\ref{subspace_ass}\ref{subspace_ass:A1} (Smoothness \& invertibility).}
Assumption~\ref{subspace_ass}\ref{subspace_ass:A1} requires that (i) each single-modality generator $g_i$ is smooth and invertible,
and (ii) the joint map $g_{ij}:(\vec{z}^{(i)}_c,\vec{z}^{(i)}_s,\tilde{\vec{z}}^{(j)}_{\setminus i})\mapsto(\vec{x}^{(i)},\vec{x}^{(j)})$
is also smooth and invertible.
Conceptually, this assumption encodes a standard {information sufficiency} requirement:
if a modality discards information about its latents (non-injective mapping), then those latents cannot be recovered from observations,
and identifiability is impossible even in principle. Similar invertibility/injectivity assumptions are common in nonlinear ICA and
identifiable deep latent-variable models, including time-contrastive / auxiliary-variable based nonlinear ICA, identifiable VAEs,
and contrastive-learning based identifiability results \cite{zheng2022identifiability,yao2023multi,sun2024causal,kong2023understanding,von2021self}.

From a practical standpoint, exact global invertibility can be seen as an idealization; what is typically sufficient for theory is
{local} invertibility on the data support (i.e., the observation manifold). Moreover, many modern generative families are
explicitly invertible (e.g., flow-based models), and even when using general decoders, training objectives often encourage near-injectivity
on the relevant region of the latent space. Thus, Assumption~\ref{subspace_ass}\ref{subspace_ass:A1} is a standard and broadly adopted
regularity condition for identifiability analysis.

\paragraph{Assumption~\ref{subspace_ass}\ref{subspace_ass:A2} (Collective linear independence across neighbors).}
Assumption~\ref{subspace_ass}\ref{subspace_ass:A2} does {not} require any single pair $(i,j)$ to fully determine the shared block
$\vec{z}^{(i)}_c$. Instead, it only requires that the {collection} of neighbor responses jointly spans all directions of $\vec{z}^{(i)}_c$,
formalized by the existence of matrices $L_{ij}$ such that $\sum_{j\in\mathcal{N}(i)} L_{ij}A_{j\leftarrow i}=I(\vec{z}^{(i)}_c)$.
This is exactly the “multi-view complementarity” phenomenon: different modalities often capture different aspects of the same underlying
semantics, and each neighbor may reveal only a subset of factors, but together they cover the whole shared subspace. Closely related
collective-rank conditions also appear in multi-view / multimodal identifiability analyses \cite{yao2023multi,sun2024causal}.

In real datasets, Assumption~\ref{subspace_ass}\ref{subspace_ass:A2} is typically plausible when:
(i) the shared dimension $d_c^{(i)}$ is not larger than the effective total neighbor information
(e.g., $\sum_{j\in\mathcal N(i)} d_x^{(j)}$ in a local linear sense),
and (ii) different neighbors react to different directions of the shared factors (i.e., their Jacobians are not all aligned).
Importantly, even if each $A_{j\leftarrow i}$ is rank-deficient, the intersection of nullspaces can still be $\{0\}$, which is precisely
why pairwise but graph-connected supervision can suffice.

\subsubsection{Why the component-wise assumptions are reasonable and broadly applicable}

\paragraph{Why do we need extra assumptions at all?}
Upgrading from block identifiability to component-wise identifiability is fundamentally a {model selection} problem:
block identifiability leaves an equivalence class of solutions related by within-block diffeomorphisms.
It is well-known that without inductive bias, fully unsupervised disentanglement (component-wise recovery) is impossible in general
\cite{locatello2019challenging,scholkopf2021toward}. Therefore, to obtain a one-to-one component correspondence, we must introduce additional structure.
Our choice is a sparsity-based structural preference on cross-modality mechanisms, which is mild and interpretable.

\paragraph{Assumption~\ref{ass:pair_coverage_dedup} (Pair coverage of nonzero cross blocks).}
This assumption aligns the identifiability goal with the data availability encoded by the observation graph $\mathcal G=([M],\mathcal E)$.
After removing overlap coordinates (so that we only count dependencies between {non-overlap} parts),
Assumption~\ref{ass:pair_coverage_dedup} states that {unobserved} modality pairs should not contain additional non-overlap cross-modality
dependence that we never see in the data; otherwise, any criterion evaluated only on observed pairs would be insufficient to uniquely select
the correct component alignment.
In applications, this is reasonable when $\mathcal E$ is chosen to reflect the actual data collection design: we only aim to recover
cross-modality structure for modality pairs that are co-observed (or can be connected through co-observed chains). If some pairs are never
observed but still interact strongly in the true system, then purely observational pairwise data cannot determine those interactions,
and one should either collect additional pair data (add edges) or relax the target (settle for block-level guarantees).

\paragraph{Assumption~\ref{ass:global_sparse_minimal_dedup} (Global sparsity minimality).}
This assumption is an explicit Occam-type principle: among all block-wise invertible reparameterizations, the true representation yields the
{sparsest} cross-modality interaction pattern (on non-overlap parts), up to within-block generalized permutations.
Such “minimality/sparsity selects the correct structure” principles are widely used in causal discovery and sparse mechanism learning.
For example, classical identifiable causal models exploit structural constraints (e.g., non-Gaussianity and acyclicity in LiNGAM), and broader causal literature emphasizes minimality/simplicity as a central guiding principle
\cite{peters2017elements,lachapelle2022partial,zheng2022identifiability}. More recently, representation learning works explicitly connect {mechanism sparsity} to disentanglement and
component-level identifiability.
In multimodal settings, sparsity is often natural: only a limited number of semantic factors truly induce cross-modality dependencies, while
many variations remain modality-specific; moreover, mixing coordinates within a modality typically spreads each cross-modality influence across
multiple coordinates, increasing the number of nonzeros and thus being disfavored by the sparsity criterion~\cite{sun2024causal}.

\subsubsection{Some clarifications}

\paragraph{Q1: Is Assumption~\ref{subspace_ass}\ref{subspace_ass:A1} too strong because neural decoders are not strictly invertible?}
\paragraph{A1:}
In theory, invertibility ensures that observations preserve enough information to serve as (local) coordinates for latents.
In practice, exact global invertibility is not strictly necessary; local invertibility on the data support is often sufficient for the
Jacobian-based arguments in the proof. If the application demands strict guarantees, one can adopt explicitly invertible architectures
(e.g., flow-based generators) or injective decoders. If a modality is truly non-injective (information-destroying), then the lost factors
are not identifiable by any method, and the correct conclusion should be a weaker equivalence class.

\paragraph{Q2: Does Assumption~\ref{subspace_ass}\ref{subspace_ass:A2} require every neighbor $j$ to be fully informative about $\vec{z}^{(i)}_c$?}
\paragraph{A2:}
No. The assumption is {collective}. Each $A_{j\leftarrow i}$ can be rank-deficient; what matters is that the stacked operator across
$j\in\mathcal N(i)$ has full column rank (cf.\ Lemma~\ref{lem:equiv_collective}).
This matches practical multimodal reality: one modality may reveal some semantic directions clearly while another reveals complementary ones.

\paragraph{Q3: How many neighbors are needed for block identifiability?}
\paragraph{A3:}
A necessary condition is that the aggregated neighbor information can cover the shared dimension,
i.e., the stacked Jacobian across neighbors should have rank $d_c^{(i)}$.
Heuristically, having more diverse neighbors (heterogeneous modalities) reduces the risk that all Jacobians align in the same subspace.
This is exactly why a connected observation graph with overlapping pairwise data is powerful.

\paragraph{Q4: Why sparsity leads to component-wise (one-to-one) correspondence?}
\paragraph{A4:}
Block identifiability allows within-block mixing $h_m$.
When the true cross-modality interaction structure is sparse, a nontrivial mixing $T_m$ typically spreads one interaction direction into
many coordinates, increasing the number of nonzero cross-block entries. Therefore, selecting the sparsest admissible mechanism rules out
most mixings and leaves only generalized permutations, which imply that each coordinate depends on only one input coordinate, i.e.,
component-wise up to permutation (as formalized in Theorem~\ref{thm:component_ident_dedup}).

\paragraph{Q5: Is the sparsity minimality assumption always valid? What if the true system is dense?}
\paragraph{A5:}
If the underlying cross-modality mechanism is genuinely dense, then sparsity may not distinguish the correct component alignment, and
component-wise identifiability may fail---but block/subspace identifiability (Theorem~\ref{thm:multi_pair_ident}) can still hold.
In such cases, a different inductive bias (e.g., independence constraints, temporal structure, interventions, or supervised anchors) may be
needed to go beyond subspace-level guarantees.

\paragraph{Q6: Why do we remove overlap coordinates via $\mathcal I_{m\mid n}$ before counting sparsity?}
\paragraph{A6:}
Because the same global shared factor $c_r$ can appear in multiple modality-wise shared blocks, naively counting cross-block nonzeros would
double-count “self” relations induced by duplication and could artificially favor incorrect mixings.
The non-overlap indexing ensures the sparsity criterion focuses on {genuine} cross-modality interactions between distinct factors,
which is exactly what should constrain component alignment.

\paragraph{Q7: How should we understand the identifiability assumptions, and can MPM handle sample-wise or partially paired data?}
\paragraph{A:}
MPM provides stronger identifiability than prior settings because it does not require globally joint observations over all modalities.
Instead, it relies on the marginal distributions of observed modality pairs, so shared latent variables can be recovered from sparsely connected pairwise observations when the modality graph provides sufficient connectivity and rank information.
Assumption~2 is also collective rather than pairwise-uniform: for a given modality, the Jacobian information contributed by its paired modalities only needs to jointly form a full-rank matrix, as formalized in Lemma~1.
Thus, different modality pairs may contribute unequally to identifiability.
For example, if rank three is required and three pairs contribute ranks one, two, and one, removing the rank-two pair breaks identifiability, whereas removing a redundant rank-one pair may not.
Moreover, the framework extends to sample-wise or partially paired data because pairing is defined at the distribution level rather than requiring every sample to be fully paired.
Noisy or partial pairing can be viewed as corruption, where consistent learning remains possible when such corruption is controlled, corrected, or filtered.


\section{Additional Experimental Details}
\label{sec:appendix_b}

This appendix provides implementation and training details for reproducibility, complementing the main experimental section.
The Qwen3-Omni-30B-A3B backbone and all off-the-shelf modality encoders remain fully frozen, and we only optimize modality-side modules (the Stage~I alignment autoencoder and the Stage~II mappers).

\subsection{Training Protocol and Hyperparameters}
\label{sec:appendix_b_training}

\paragraph{Training framework.}
All training is conducted with Ms-Swift v4.0 (\texttt{ms-swift4.0}) from ModelScope, with our extensions integrated into its training pipeline (see \cref{sec:appendix_b_code}).

\paragraph{Two-stage training.}
We follow the two-stage protocol described in the main paper.
Stage~I learns the modality-side alignment autoencoder under paired supervision.
Stage~II trains the modality-to-backbone interface (recomposition) for downstream generation and/or selection, while keeping the Qwen3-Omni backbone frozen.

\paragraph{Learning rates.}
We use stage-specific learning rates:
(i) Stage~I: $\mathrm{lr}=5\times 10^{-5}$; \quad
(ii) Stage~II: $\mathrm{lr}=5\times 10^{-5}$.

\paragraph{Batch size.}
We keep the per-device batch size fixed within each stage:
(i) Stage~I: per-device batch size $=16$; \quad
(ii) Stage~II: per-device batch size $=8$.
The global batch size varies with the number of GPUs used in a given run.

\paragraph{Epochs.}
We train for:
(i) Stage~I: 10 epochs; \quad
(ii) Stage~II: 1 epoch.

\paragraph{Other hyperparameters.}
All remaining hyperparameters and engineering settings are fully specified in our released YAML configuration files and can be directly inspected in the accompanying codebase.

\subsection{Training Cost and Hardware}
\label{sec:appendix_b_cost}

\paragraph{Hardware.}
All experiments are conducted on AMD hardware.
Each MI210 server contains 8$\times$ AMD MI210 64GB GPUs.
We report wall-clock time for a single representative run for each stage and modality.

\paragraph{Stage~I cost (single server).}
Stage~I is trained on one MI210 server (8 GPUs total).
A full Stage~I run takes approximately:
(i) point cloud modality: $\sim$1 hour; \quad
(ii) tactile modality: $<1$ hour.

\paragraph{Stage~II cost (multi-server).}
Stage~II is trained on 8 MI210 servers (64 GPUs total).
A full Stage~II run takes approximately:
(i) point cloud modality: $\sim$6 hours; \quad
(ii) tactile modality: $\sim$3 hours.

\paragraph{Remark on MoE efficiency.}
Our backbone, Qwen3-Omni-30B-A3B, is a Mixture-of-Experts (MoE) model.
Because our compute cluster consists solely of AMD GPUs, we cannot use training stacks that are primarily engineered for NVIDIA GPUs (e.g., Megatron-style optimized MoE kernels and systems), leading to significantly reduced training efficiency for MoE backbones in practice.
Notably, this overhead persists even though the backbone parameters are frozen, since Stage~II still requires running the backbone graph to obtain gradients with respect to the modality-conditioned soft embeddings.
We expect that training on a comparable NVIDIA cluster with mature MoE-optimized kernels/frameworks would substantially reduce the overall training cost.

\subsection{Point-Cloud Data Augmentation}
\label{sec:appendix_b_aug}

To improve generalization on 3D point clouds, we apply data augmentation during training.

\paragraph{Random rotations.}
We augment point clouds with random rotations where rotation angles are constrained to be multiples of $90^\circ$.
In the current experimental setting, we generate an additional rotated set whose size matches the original training set (i.e., 100\% extra data volume for rotation augmentation).

\paragraph{Color-to-gray conversion.}
To reduce over-reliance on color cues and improve robustness, we additionally apply a color transformation to a subset of point clouds by converting colors to pure gray.
When this augmentation is applied, we also update the corresponding caption by replacing the original color words with ``gray'' to maintain consistency between geometry/color and text.
In our experiments, we generate this color-transformed augmentation for 10\% of the training data volume.

\paragraph{Consistency with baselines.}
Our Adapter baseline uses {exactly the same training data} as MPM and therefore employs the same augmentation pipeline (rotation and color-to-gray) for fair comparison.

\subsection{Codebase and Implementation Notes}
\label{sec:appendix_b_code}

\paragraph{Ms-Swift integration.}
We implement MPM by extending the ModelScope Ms-Swift codebase with a plugin-style design.
This enables convenient customization of the LLM training pipeline (e.g., inserting our modality-side autoencoder and Stage~II mappers) while avoiding intrusive modifications to the upstream framework.

\paragraph{Modality-specific modules.}
All newly introduced architectures and utilities are modularized by modality:
\begin{itemize}
  \item 3D point cloud components: stored under \texttt{/point\_clouds}.
  \item Tactile components: stored under \texttt{/tactile}.
\end{itemize}
Importantly, we do not modify Ms-Swift core code; instead, we register and invoke our modules through the framework's extension interfaces.

\paragraph{Documentation.}
Additional implementation instructions (environment setup, entrypoints, and configuration usage) are provided in the submitted codebase \texttt{README.md}.

\begin{figure*}[t]
\centering
\setlength{\fboxsep}{3.5pt}
\setlength{\fboxrule}{0.4pt}

\qualpanel{(a) ModelNet40}{Classification}{
  \pcview{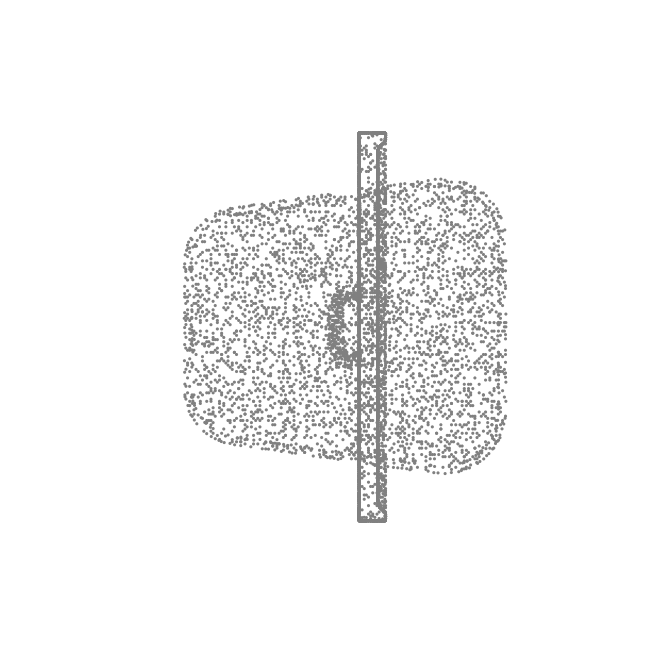}{trim=30 30 30 30,clip}}{
  \pcview{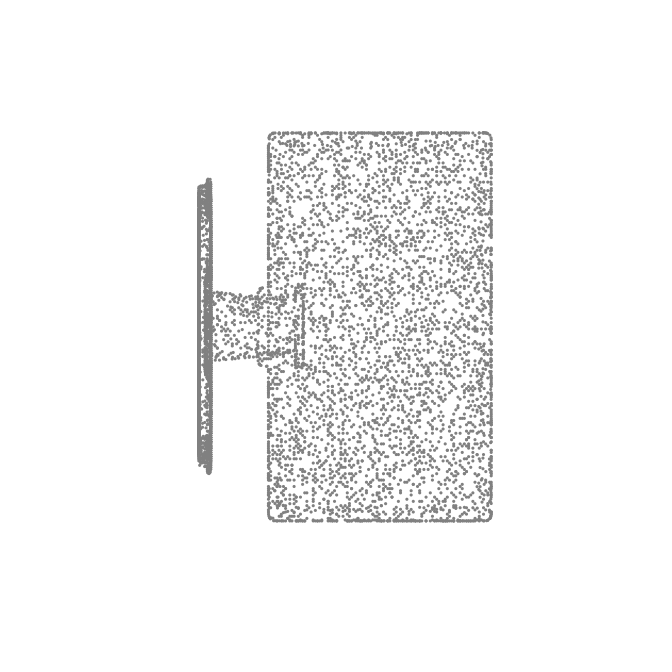}{trim=30 30 30 30,clip}}{
  \pcview{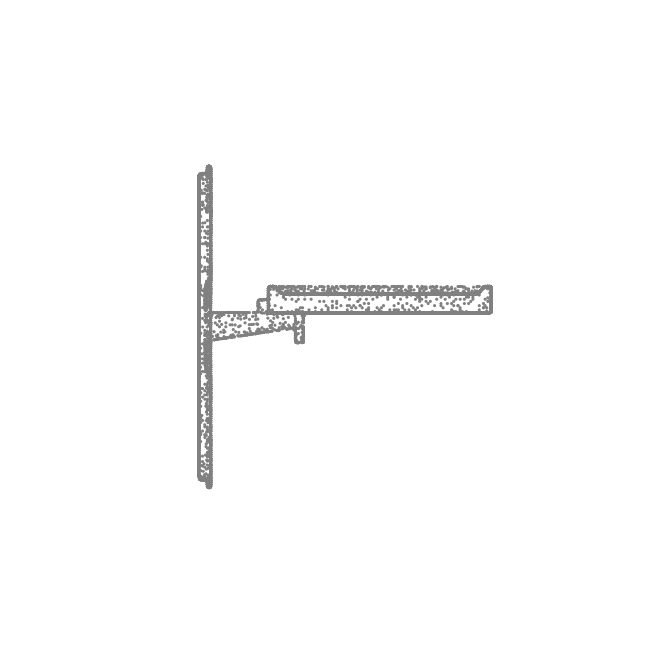}{trim=30 30 30 30,clip}}{
  Classify the object into exactly one of the following ModelNet40 categories: \{40 labels\}.}{
  \texttt{monitor}}{
  \texttt{table}.}{
  \texttt{monitor}.}
\hfill
\qualpanel{(b) ModelNet40}{Classification}{
  \pcview{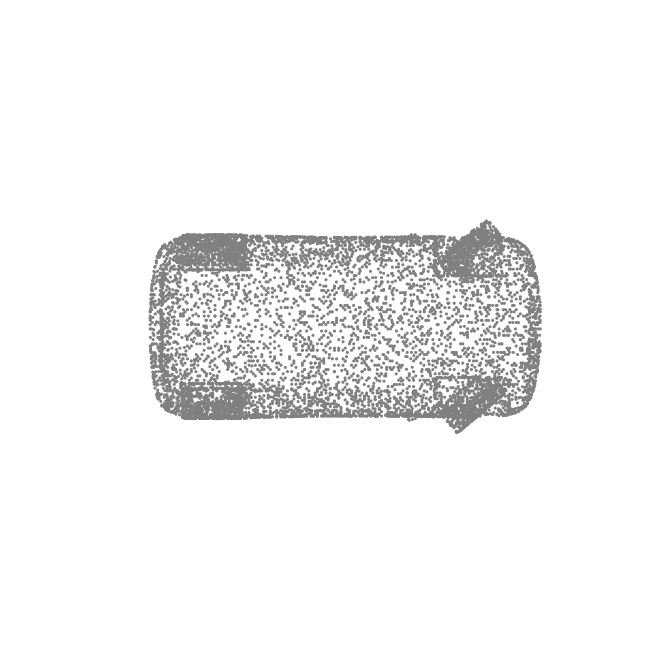}{trim=40 40 40 40,clip}}{
  \pcview{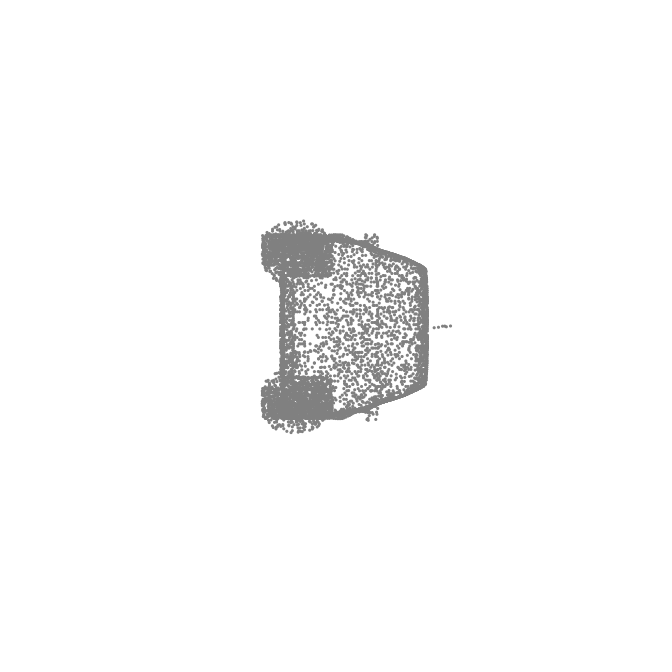}{trim=40 40 40 40,clip}}{
  \pcview{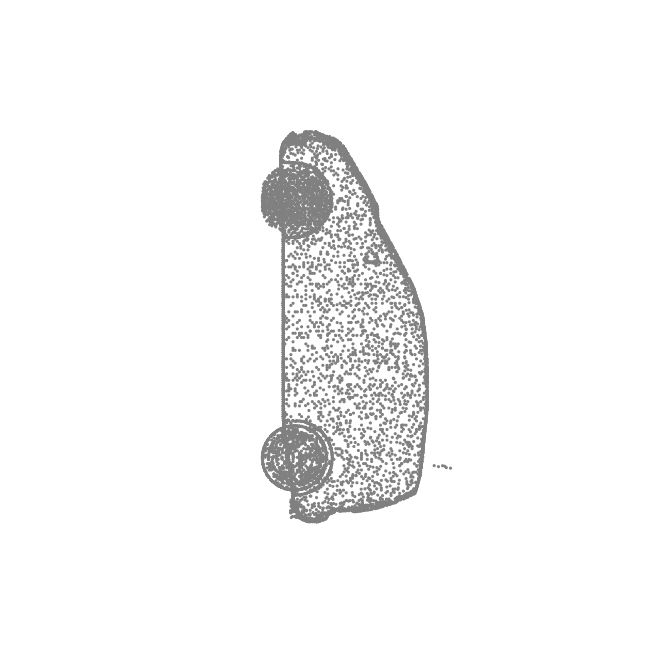}{trim=40 40 40 40,clip}}{
  Classify the object into exactly one of the following ModelNet40 categories: \{40 labels\}.}{
  \texttt{car}}{
  \texttt{bathlib}.}{
  \texttt{car}.}

\par\vspace{0.9em}

\qualpanelpad{2.6\baselineskip}
\qualpanel{(c) 3D MM-VET}{Open-ended QA}{
  \pcview{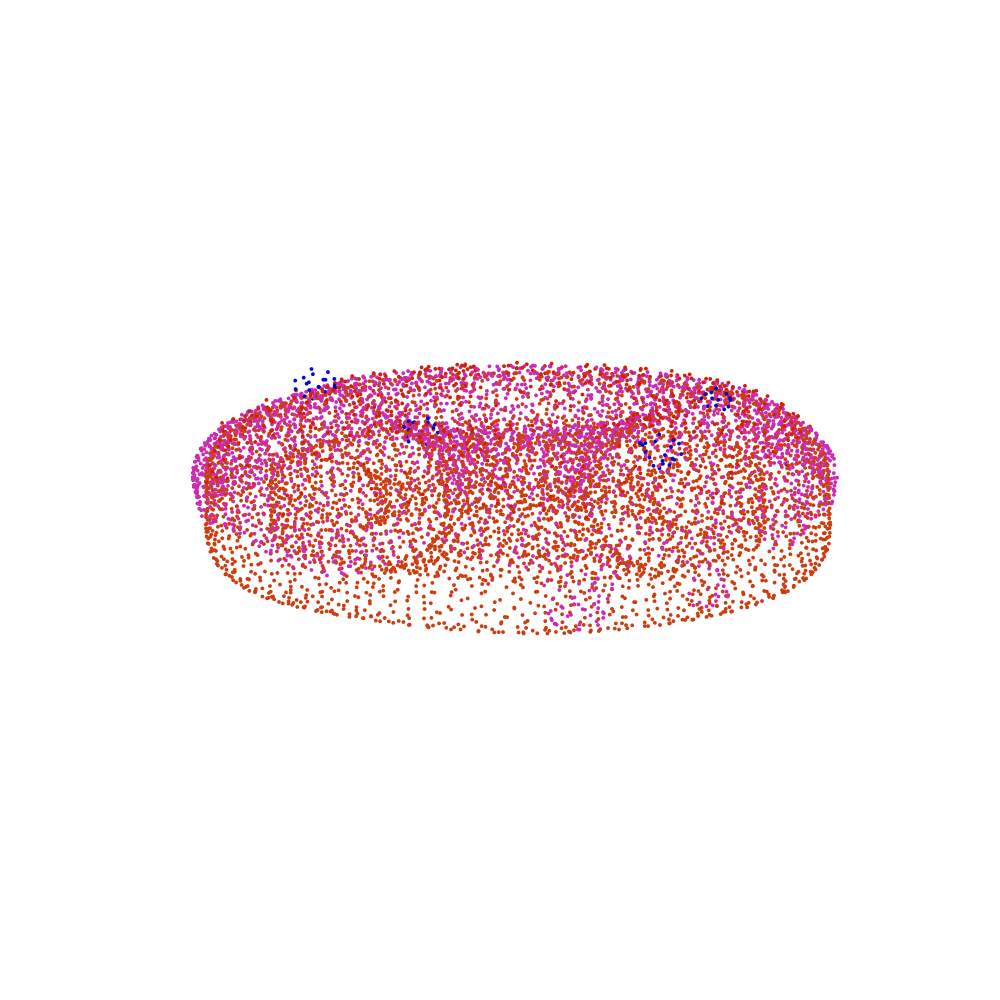}{trim=49 99 41 98,clip}}{
  \pcview{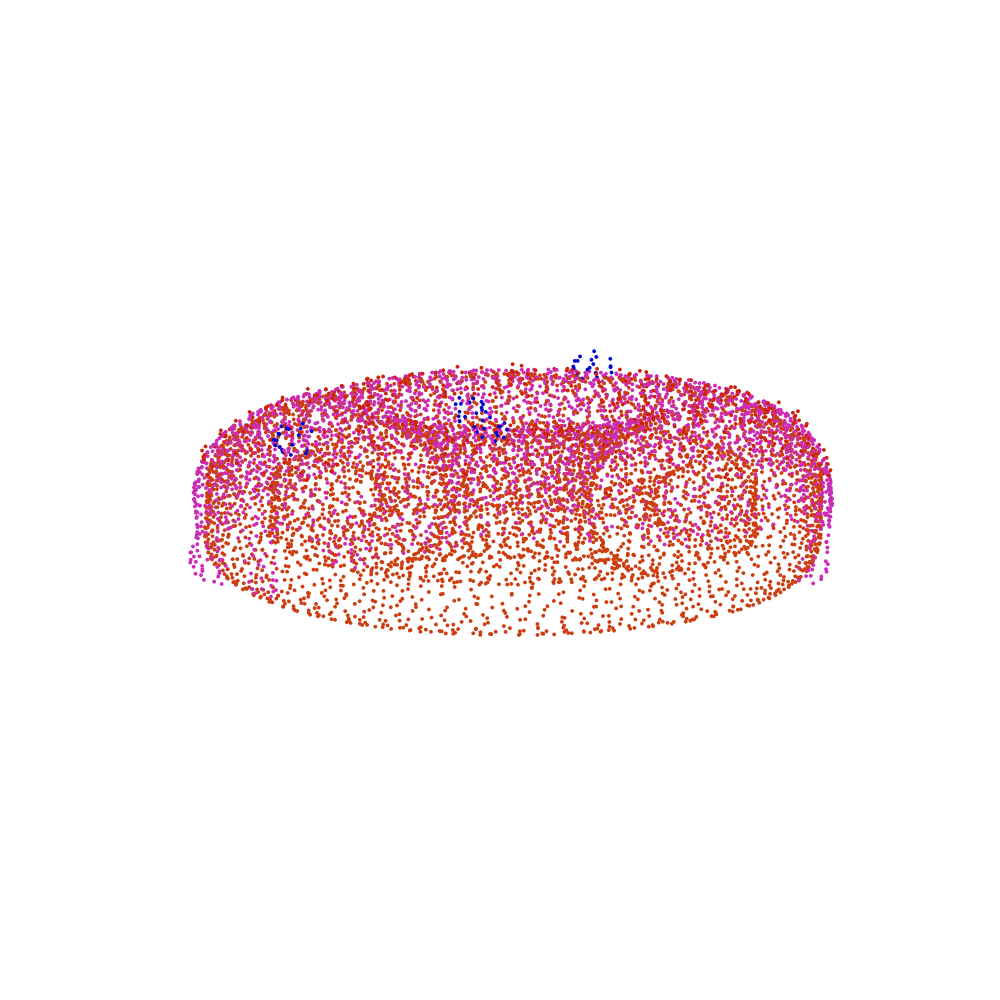}{trim=48 99 42 95,clip}}{
  \pcview{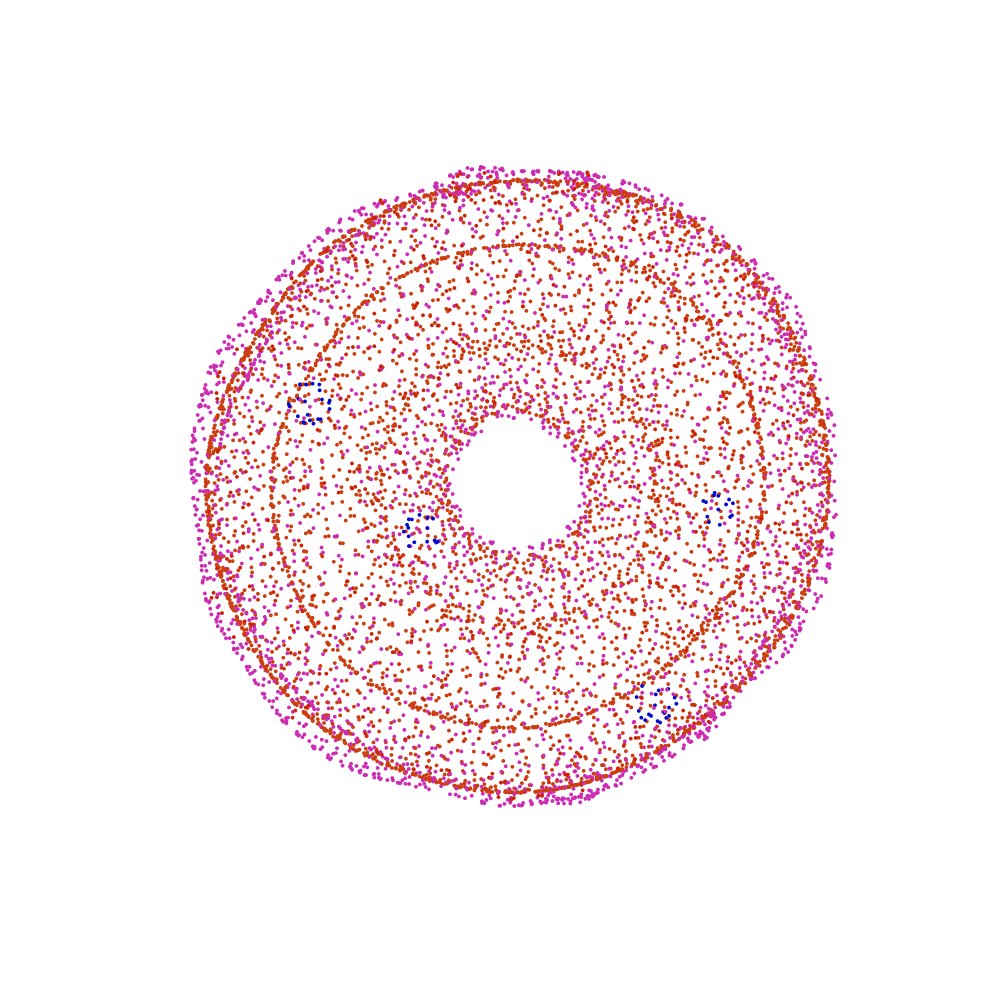}{trim=49 50 41 42,clip}}{
  What is this object mainly made of?}{
  A pink and purple donut with sprinkles.}{
  Based on the visual information provided, it is impossible to determine the material composition of the object.}{
  This is a doughnut, which is mainly made of dough, sugar and cream.}
\hfill
\qualpanel{(d) 3D MM-VET}{Open-ended QA}{
  \pcview{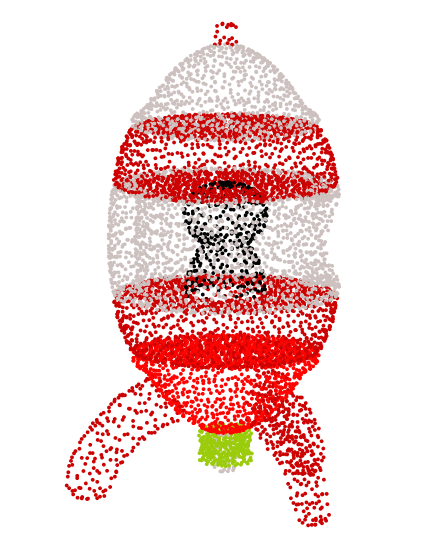}{}}{
  \pcview{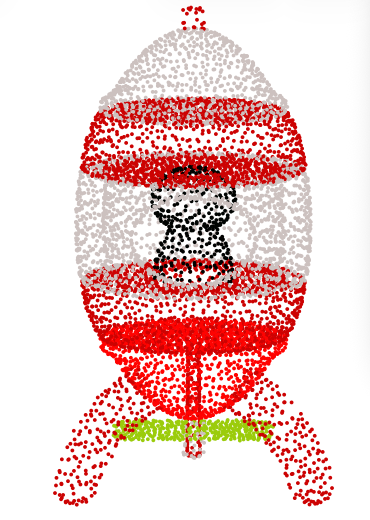}{trim=0 0 20 0,clip}}{
  \pcview{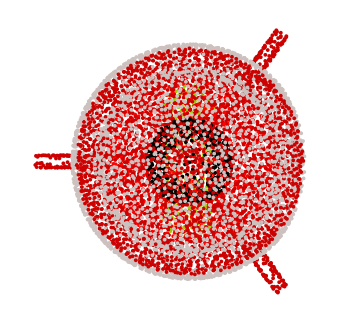}{}}{
  Please describe how important it is to a country.}{
  A small red and white rocket-shaped toy with a handle.}{
  The object depicted is the national flag of Libya, which is a symbol of national identity and sovereignty.}{
  A rocket symbolizes a nation's technological prowess, military strength, and scientific advancement. It enables space exploration...}

\vspace{-0.4em}
\caption{\textbf{Qualitative examples on point-cloud benchmarks.}
We visualize each point cloud using three rendered views (front/side/top) for readability.
The \textbf{baseline} feeds these renders to the frozen Qwen3-Omni backbone, whereas \textbf{MPM} consumes point-cloud tokens directly (the renders are shown for visualization and are not required by MPM).
(a,b) are from ModelNet40 classification; (c,d) are from 3D MM-VET open-ended QA.}
\label{fig:appendix_point_qual}
\vspace{-0.6em}
\end{figure*}

\subsection{More Experiments}
\label{sec:more_experiments}

\noindent\textbf{Controlled comparison under the same backbone.}
To further isolate the effect of the proposed pairwise latent alignment and recomposition design, we conduct a controlled comparison using the same Qwen3-Omni-30B-A3B backbone.
As shown in \cref{tab:qwen3_omni_backbone}, MPM consistently outperforms the PointLLM-style baseline under the unified backbone setting, suggesting that the gains are not solely due to the choice of the LLM backbone.

\begin{table}[t]
  \centering
  \caption{\textbf{Controlled comparison under the same Qwen3-Omni-30B-A3B backbone.}}
  \label{tab:qwen3_omni_backbone}
  \small
  \setlength{\tabcolsep}{6pt}
  \renewcommand{\arraystretch}{1.0}
  \begin{tabular}{lcc}
    \toprule
    \textbf{Method} & \textbf{3D MM-VET} & \textbf{ModelNet40} \\
    \midrule
    PointLLM-style baseline & 54.5 & 60.3 \\
    MPM                    & \textbf{57.6} & \textbf{70.5} \\
    \bottomrule
  \end{tabular}
  \vspace{-0.3cm}
\end{table}

\vspace{0.2cm}
\noindent\textbf{Extension to infrared modality.}
To examine whether MPM can generalize beyond point clouds and tactile sensing, we further extend it to infrared imagery.
We evaluate on the Infrared Template Benchmark, which includes aerial counting, pedestrian counting, recognition, and safety-related tasks.
As shown in \cref{tab:infrared_template}, MPM achieves the best results on aerial counting and pedestrian counting, and remains competitive on recognition and safety, demonstrating its flexibility for incorporating additional modalities.

\begin{table}[t]
  \centering
  \caption{\textbf{Evaluation on the Infrared Template Benchmark.}}
  \label{tab:infrared_template}
  \small
  \setlength{\tabcolsep}{4pt}
  \renewcommand{\arraystretch}{1.0}
  \resizebox{\linewidth}{!}{
    \begin{tabular}{lcccc}
      \toprule
      \textbf{Model} & \textbf{Aerial Count} & \textbf{Pedestrian Counting} & \textbf{Recognition} & \textbf{Safety} \\
      \midrule
      ImageBindLLM-7B   & 18.87 & 50.89 & 30.32 & 77.38 \\
      PandaGPT-7B       & 16.19 & 16.87 & 39.82 & 66.97 \\
      Infrared-LLaVA-7B & 35.63 & 85.50 & \textbf{95.85} & \textbf{94.54} \\
      MPM               & \textbf{45.38} & \textbf{90.36} & \underline{86.90} & \underline{93.06} \\
      \bottomrule
    \end{tabular}
  }
  \vspace{-0.3cm}
\end{table}

\vspace{0.2cm}
\noindent\textbf{Sensitivity analysis.}
We further study the sensitivity of MPM to the shared-mask size and the reconstruction loss weight $\lambda$ on TVL Stage~I.
As shown in \cref{tab:mask_lambda_ablation_tvl_stage1}, MPM maintains stable performance under different mask sizes and loss weights, indicating that the method is not overly sensitive to these hyperparameters.

\begin{table}[t]
  \centering
  \caption{\textbf{Sensitivity analysis on TVL Stage~I.}}
  \label{tab:mask_lambda_ablation_tvl_stage1}
  \small
  \setlength{\tabcolsep}{3pt}
  \renewcommand{\arraystretch}{1.0}

  \begin{minipage}[t]{0.49\linewidth}
    \centering
    \begin{tabular}{ccc}
      \toprule
      \textbf{3-shared} & \textbf{2-shared} & \textbf{Score} \\
      \midrule
      400 & 112 & 4.84 \\
      800 & 224 & 4.65 \\
      200 & 56  & 4.71 \\
      \bottomrule
    \end{tabular}

    \vspace{2pt}
    \footnotesize{(a) Mask size ablation.}
  \end{minipage}
  \hfill
  \begin{minipage}[t]{0.45\linewidth}
    \centering
    \begin{tabular}{cc}
      \toprule
      $\boldsymbol{\lambda}$ & \textbf{Score} \\
      \midrule
      0.1 & 4.84 \\
      0.3 & 4.68 \\
      0.5 & 4.61 \\
      \bottomrule
    \end{tabular}

    \vspace{2pt}
    \footnotesize{(b) Loss-weight ablation.}
  \end{minipage}
  \vspace{-0.3cm}
\end{table}

\vspace{0.2cm}
\noindent\textbf{Effect of asymmetric alignment.}
We also evaluate the role of asymmetric alignment in Stage~I.
As shown in \cref{tab:asym_ablation}, removing asymmetric alignment leads to a clear performance drop, suggesting that this component helps better align modality-specific shared codes under pairwise-only supervision.

\begin{table}[t]
  \centering
  \caption{\textbf{Effect of asymmetric alignment on TVL Stage~I.}}
  \label{tab:asym_ablation}
  \small
  \setlength{\tabcolsep}{6pt}
  \renewcommand{\arraystretch}{1.0}
  \begin{tabular}{lc}
    \toprule
    \textbf{Variant} & \textbf{Score} \\
    \midrule
    Full alignment        & 4.84 \\
    Asymmetric alignment & 4.52 \\
    \bottomrule
  \end{tabular}
  \vspace{-0.3cm}
\end{table}

\vspace{0.2cm}
\noindent\textbf{Stability of GPT-based evaluation.}
Since tactile evaluation relies on GPT-based scoring, we additionally measure the standard deviation across repeated evaluations.
The deviations are small on all three tactile subsets: $0.109$ on SSVTP, $0.037$ on HCT, and $0.045$ on TVL, indicating that the reported GPT-based scores are stable.

\subsection{Qualitative Results}
\label{sec:appendix_b_qualitative}

\paragraph{Discussion.}
Projecting 3D point clouds into a small set of rendered views inevitably discards information, especially fine-grained geometry and global spatial structure.
Consequently, directly applying an image-based MLLM to three-view renders can be brittle: the model may overfit to a single viewpoint and fail to integrate multi-view evidence.
For instance, in \cref{fig:appendix_point_qual}(b), Qwen3-Omni predicts \texttt{bathlib} by relying primarily on the frontal appearance and under-utilizing the other views, which contain discriminative cues for recognizing a \texttt{car}.
In \cref{fig:appendix_point_qual}(d), it further confuses a 3D rocket-like object with a 2D national flag, reflecting limited capability to infer 3D structure from sparse viewpoints.
In contrast, MPM operates on point-cloud-native representations and is more reliable on recognition-style tasks (\eg, (a,b)).
Nevertheless, MPM still exhibits imperfect instruction following on open-ended questions (\eg, (c,d)), which we attribute to training with caption-style supervision only.
A natural next step is to incorporate instruction-following data to improve generalization to broader question answering settings.

\section{More Ablation studies}


Table~\ref{tab:tactile_ablation} reports ablations on tactile understanding across SSVTP, HCT, and TVL under the same evaluation protocol (GPT-4 based scoring on a 1--10 scale). In all settings, the Qwen3-Omni backbone and the tactile encoder are kept frozen, and only modality-side modules are optimized. The full model achieves 6.24/5.09/5.22 on SSVTP/HCT/TVL, improving over the frozen Qwen3-Omni baseline (5.30/4.28/4.40) by 0.94/0.81/0.82 points, which confirms that the proposed modality-side training can inject tactile capability without updating the backbone.

Removing Stage~I and training only the Stage~II interface (recomposition-only) yields 5.56/5.00/5.06. This corresponds to drops of 0.68/0.09/0.16 compared with the full model, while still outperforming the frozen backbone by 0.26/0.72/0.66. The result indicates that paired tactile supervision can directly train a tactile-to-backbone conditioning interface that already provides a meaningful signal. However, the consistent gap to the full model, particularly on SSVTP, suggests that learning a structured latent space in Stage~I provides additional regularization and improves robustness for more fine-grained tactile attribute understanding.

Removing Stage~II while keeping Stage~I (Stage~I only) leads to 5.99/4.70/4.84, with drops of 0.25/0.39/0.38 relative to the full model. This variant remains above the frozen backbone by 0.69/0.42/0.44, implying that Stage~I can learn a useful aligned representation from paired tactile supervision. Nevertheless, the larger degradation on HCT and TVL supports the role of Stage~II as a compatibility bridge that maps the learned latent code into the backbone’s input embedding space, allowing the frozen decoder to contribute its language generation prior during tactile description.

The Stage~I objectives play distinct roles. When the reconstruction loss is removed in Stage~I, performance collapses to 2.98/3.94/3.08, far below the full model and also below the frozen backbone on all splits. This behavior indicates that reconstruction is essential for enforcing an information-preserving latent code and preventing degenerate solutions under limited paired tactile supervision. In contrast, removing the contrastive alignment loss in Stage~I produces 5.96/4.79/4.93, which is consistently worse than the full model by 0.28/0.30/0.29 but still stronger than the frozen backbone by 0.66/0.51/0.53. This suggests that contrastive alignment provides a complementary cross-modal semantic signal, while reconstruction is the dominant stabilizer that maintains latent fidelity; both contribute to the best overall tactile performance.

Overall, the ablations demonstrate that Stage~I and Stage~II are complementary and that the reconstruction term in Stage~I is critical for learning a stable tactile representation. Stage~I improves the learnability and robustness of tactile semantics under pairwise supervision, while Stage~II is necessary to fully exploit the frozen MLLM decoder for language-grounded tactile description, yielding the strongest results across SSVTP, HCT, and TVL.

\begin{table}[t]
\centering
\caption{\textbf{Ablation study on our evaluation suite.}}
\vspace{-0.2cm}
\label{tab:tactile_ablation}
\setlength{\tabcolsep}{7pt}
\renewcommand{\arraystretch}{1.15}
\begin{tabular}{lccc}
\hline
Method & SSVTP & HCT & TVL \\
\hline

Full                   &  \textbf{6.24}  & \textbf{5.09}   & \textbf{5.22} \\

w/o Stage~I (recomposition-only)          & 5.56 & 5.00 & 5.06  \\

w/o Stage~II (Stage~I only)                & 5.99 & 4.70 & 4.84 \\

w/o reconstruction $\mathcal{L}_{\mathrm{rec}}$ in Stage~I   & 2.98 & 3.94 & 3.08 \\

w/o contrastive $\mathcal{L}_{\mathrm{align}}$ in Stage~I                   &  5.96 &4.79 & 4.93    \\
Qwen3-Omni             & 5.30 & 4.28 & 4.40 \\

\hline
\end{tabular}
\vspace{-0.4cm}
\end{table}


\section{Prompt Design}
\label{sec:appendix_b_prompts}

We follow the native \emph{Qwen3-Omni} chat formatting with explicit \texttt{system} and \texttt{user} messages.
As shown in Figure~\ref{fig:appendix_prompts_point}, for the point-cloud modality, the object is provided through a dedicated text section
\texttt{3D\_POINT\_CLOUD\_EMBEDDING} placed inside the user message.
Tokens inside this section are \emph{placeholders}: their continuous embeddings are injected at training/inference time to carry semantic information of the input point cloud.
Below we summarize the exact prompts used for (i) training and (ii) ModelNet40 evaluation; additional prompt variants are available in the released code.

\lstdefinestyle{mpmprompt}{
  basicstyle=\ttfamily\scriptsize,
  columns=fullflexible,
  breaklines=true,
  breakatwhitespace=false,
  frame=single,
  framerule=0.35pt,
  rulecolor=\color{black!25},
  backgroundcolor=\color{black!2},
  xleftmargin=0.8em,
  xrightmargin=0.8em,
  aboveskip=0.45em,
  belowskip=0.25em,
  keepspaces=true,
  showstringspaces=false,
  tabsize=2
}

\begin{figure*}[t]
\centering
\vspace{-0.2em}

\begin{minipage}[t]{0.98\linewidth}
\raggedright
{\small\textbf{Chat wrapper (Qwen3-Omni template used by Ms-Swift).}}
\begin{lstlisting}[style=mpmprompt]
<|im_start|>system
{SYSTEM_PROMPT}
<|im_end|>
<|im_start|>user
{USER_QUERY}   # a.k.a. QUERY / user_text
<|im_end|>
<|im_start|>assistant
# generation starts here (add_generation_prompt=True)
\end{lstlisting}
\end{minipage}

\vspace{0.35em}

\begin{minipage}[t]{0.59\linewidth}
\raggedright
{\small\textbf{(a) Training system prompt (\texttt{DEFAULT\_SYSTEM\_PROMPT}).}}
\begin{lstlisting}[style=mpmprompt]
You are Qwen, a virtual human developed by the Qwen Team, Alibaba Group, capable of understanding text inputs and generating helpful responses.

Task setting:
- You will answer questions about an object represented by a 3D point cloud.
- In some requests, the user message will contain a section named '3D_POINT_CLOUD_EMBEDDING'. The tokens in that section are placeholders whose embeddings are injected at inference/training time to carry semantic information about the 3D object.

Instructions:
- Use the '3D_POINT_CLOUD_EMBEDDING' section as object context to answer the question.
- If the embedding section is absent, answer based only on the text question and be explicit about uncertainty.
- Output only the final answer text (no role labels such as user/assistant, no extra dialogue markers).
\end{lstlisting}
\end{minipage}
\hfill
\begin{minipage}[t]{0.39\linewidth}
\raggedright
{\small\textbf{(b) Training user message (\texttt{\{USER\_QUERY\}} skeleton).}}
\begin{lstlisting}[style=mpmprompt]
{TASK_INSTRUCTION / QUESTION}

3D_POINT_CLOUD_EMBEDDING:
<pc_tok_1> <pc_tok_2> ... <pc_tok_L>
\end{lstlisting}
{\scriptsize
\noindent\emph{Note:} \texttt{<pc\_tok\_i>} are placeholder tokens whose embeddings are injected from the point-cloud encoder + mapper at runtime.}
\end{minipage}

\vspace{0.35em}

\begin{minipage}[t]{0.59\linewidth}
\raggedright
{\small\textbf{(c) ModelNet40 system prompt (\texttt{SYSTEM\_PROMPT\_CLS}).}}
\begin{lstlisting}[style=mpmprompt]
You are Qwen, a virtual human developed by the Qwen Team, Alibaba Group.

Task setting:
- You will classify an object represented by a 3D point cloud into one of the ModelNet40 categories.
- The user message may contain a section named '3D_POINT_CLOUD_EMBEDDING'. The tokens inside that section are placeholders whose embeddings are injected at inference time to carry semantic information about the 3D object.

Instructions:
- Use the embedding section as the ONLY object context.
- Output exactly ONE category label from the provided list.
- Output only the label string, no extra words, no punctuation, no JSON.
\end{lstlisting}
\end{minipage}
\hfill
\begin{minipage}[t]{0.39\linewidth}
\raggedright
{\small\textbf{(d) ModelNet40 user message (\texttt{\{user\_text\}} skeleton).}}
\begin{lstlisting}[style=mpmprompt]
Classify the object into exactly one of the following ModelNet40 categories: {40 labels}.

3D_POINT_CLOUD_EMBEDDING:
<pc_tok_1> <pc_tok_2> ... <pc_tok_L>
\end{lstlisting}
{\scriptsize
\noindent\emph{Output constraint:} the model must emit exactly one label token string (no punctuation or extra words).}
\end{minipage}

\vspace{-0.5em}
\caption{\textbf{Prompt templates for point-cloud training and evaluation.}
We use the Qwen3-Omni chat wrapper from Ms-Swift, consisting of a \texttt{system} message and a \texttt{user} message, followed by an \texttt{assistant} generation prefix.
Point-cloud information is injected via placeholder tokens inside the \texttt{3D\_POINT\_CLOUD\_EMBEDDING} block.
For open-ended tasks (training and 3D MM-VET), we adopt \texttt{DEFAULT\_SYSTEM\_PROMPT};
for ModelNet40, we use a stricter classification-oriented system prompt that enforces one-label outputs.}
\label{fig:appendix_prompts_point}
\vspace{-0.6em}
\end{figure*}

\paragraph{Implementation note.}
In code, we construct a two-turn conversation
(\texttt{system}: prompt, \texttt{user}: query) and call
\texttt{processor.apply\_chat\_template(...)}
to produce the final token sequence consumed by the frozen Qwen3-Omni backbone.
Full prompt templates and all task-specific variants are provided in our released implementation.

\end{document}